\documentclass[11pt, a4paper, oneside,reqno]{amsart}

\makeatletter
\def\@settitle{\begin{center}%
		\baselineskip14\p@\relax
		\normalfont\LARGE\bfseries
		\@title
	\end{center}%
}

\def\section{\@startsection{section}{1}%
	\z@{.7\linespacing\@plus\linespacing}{.5\linespacing}%
	{\normalfont\large\bfseries}}

\def\subsection{\@startsection{subsection}{2}%
	\z@{.5\linespacing\@plus.7\linespacing}{.5\linespacing}%
	{\normalfont\bfseries}}

\def\@setauthors{%
  \begingroup
  \def\thanks{\protect\thanks@warning}%
  \trivlist
  \centering\footnotesize \@topsep30\p@\relax
  \advance\@topsep by -\baselineskip
  \item\relax
  \author@andify\authors
  \def\\{\protect\linebreak}%
  \authors%
  \ifx\@empty\contribs
  \else
    ,\penalty-3 \space \@setcontribs
    \@closetoccontribs
  \fi
  \endtrivlist
  \endgroup
}

\makeatother

\usepackage[square,numbers]{natbib}
\usepackage{xcolor}
\usepackage{graphicx}
\usepackage{hyperref}
\hypersetup{hidelinks}
\usepackage{multirow} 
\usepackage{bbm}
\usepackage{tikz-cd}
\usepackage[ruled,linesnumbered, noend]{algorithm2e}
\usepackage{booktabs}
\usepackage{multirow}
\usepackage[font=small,labelfont=rm]{subcaption}
\usepackage[font=small,margin=10pt]{caption}
\usepackage{amsmath,amssymb,amsfonts}
\usepackage{bm,mathtools}

\def\borel{{\mathcal{B}}}
\def\expect{{\mathbb{E}}}

\def\Prob{\mathbb{P}}
\def\ProbHat{\widehat{\mathbb{P}}}

\def\ProbQ{\mathbb{Q}}
\def\ProbD{\mathbb{D}}

\def\samplesVector{\mathcal{D}_N}

\def\wasserstein{{\mathbb{W}}}

\def\indicator{{\mathbbm{1}}}

\newcommand{\elem}[2]{{#1^{{(#2)}}}}

\newcommand{\norm}[1]{\left\lVert#1\right\rVert}

\def\realNum{{\mathbb{R}}}

\def\natNum{\mathbb{N}}

\def\sSimplex{\Delta}

\def\sB{{\mathcal{B}}}
\def\sC{{\mathcal{C}}}
\def\sD{{\mathcal{D}}}
\def\sE{{\mathcal{E}}}

\def\sN{{\mathcal{N}}}

\def\sP{{\mathcal{P}}}

\def\sT{{\mathcal{T}}}

\def\sX{{\mathcal{X}}}

\def\vgamma{{{\gamma}}}

\def\vpi{{{\pi}}}

\def\vomega{{{\omega}}}

\def\vc{{{c}}}

\def\vv{{{v}}}

\def\vx{{{x}}}

\def\vz{{{z}}}

\newcommand{\evpi}[1]{{\vpi^{(#1)}}}

\newcommand{\evomega}[1]{{\vomega^{(#1)}}}

\newcommand{\evgamma}[1]{{\vgamma^{(#1)}}}

\newcommand{\evv}[1]{{\vv^{(#1)}}}

\newcommand{\evz}[1]{{\vz^{(#1)}}}

\def\lowerb{p_{\ell}}
\newcommand{\evlowerb}[1]{{\lowerb^{(#1)}}}
\def\upperb{p_{u}}
\newcommand{\evupperb}[1]{{\upperb^{(#1)}}}

\newcommand{\constrainedSimplex}[2]{\sSimplex_{[#1, #2]}}

\newcommand{\changed}[1]{{#1}}

\usepackage{amsthm}
    
\theoremstyle{plain}
\newtheorem{theorem}{Theorem}[section]
\newtheorem{proposition}[theorem]{Proposition}

\newtheorem{lemma}[theorem]{Lemma}

\newtheorem{remark}{Remark}

\newif\ifdoublecolumn
\doublecolumnfalse
\newcommand{\maybenewline}{\ifthenelse{\boolean{doublecolumn}}{\\}{}}
\newcommand{\maybeamp}{\ifthenelse{\boolean{doublecolumn}}{&}{}}
\newcommand{\maybeqquad}{\ifthenelse{\boolean{doublecolumn}}{\qquad}{}}
\newcommand{\maybenonumber}{\ifthenelse{\boolean{doublecolumn}}{\nonumber}{}}

\title{Efficient Distribution Learning with Error Bounds in Wasserstein Distance}

\author{Eduardo Figueiredo$^{*}$, Steven Adams$^{*}$, Luca Laurenti$^{**}$
}
\thanks{$^*$ Equal contribution. The authors are with the Delft Center for
Systems and Control, Delft University of Technology, Delft, The Netherlands.$^{**}$ is also with AI4I, Turin, Italy. Corresponding author’s email: e.figueiredo@tudelft.nl. This research is partially supported by the NWO (grant
OCENW.M.22.056)}

\begin{document}
\maketitle

\begin{abstract}
    
The Wasserstein distance has emerged as a key metric to quantify distances between probability distributions, with applications in various fields, including machine learning, control theory, decision theory, and biological systems. Consequently, learning an unknown distribution with non-asymptotic and easy-to-compute error bounds in Wasserstein distance has become a fundamental problem in many fields.   
In this paper,  we devise a novel algorithmic and theoretical framework to approximate an unknown probability distribution $\Prob$ from a finite set of samples by an approximate discrete distribution $\ProbHat$ while bounding the Wasserstein distance between $\Prob$ and $\ProbHat$. Our framework leverages optimal transport, nonlinear optimization, and concentration inequalities.  In particular, we show that, even if $\Prob$ is unknown, the Wasserstein distance between $\Prob$ and $\ProbHat$ can be efficiently bounded with high confidence by solving a tractable optimization problem (a mixed integer linear program) of a size that only depends on the size of the support of $\ProbHat$. This enables us to develop intelligent clustering algorithms to optimally find the support of $\ProbHat$ while minimizing the Wasserstein distance error. On a set of benchmarks, we demonstrate that our approach outperforms state-of-the-art comparable methods by generally returning approximating distributions with substantially smaller support and tighter error bounds.   

\end{abstract}


\section{Introduction}\label{section:introduction}
Given a vector of $N$ independent and identically distributed (i.i.d.) samples $\samplesVector=(\vx_1,...,\vx_N)$ from an unknown distribution $\Prob$, learning a distribution $\ProbHat$ approximating $\Prob$ from $\samplesVector$ represents a fundamental problem in many fields, ranging from machine learning and control theory to biological systems and statistics \citep{campi2009scenario,karaman2011sampling,badings2022sampling,echeveste2020cortical}. While many choices are possible
, a case of particular importance, both theoretical and practical, is when $\ProbHat$ is assumed to be a discrete distribution, e.g., the empirical distribution \citep{fournier2015rate,weed2019sharp,mohajerin2018data,kuhn2019wasserstein}. This is motivated by the computational tractability of this setting and by the consistency and simplicity of the resulting discrete distributions, which are particularly suitable to be analyzed or propagated through possibly non-linear functions \citep{figueiredo2025efficient}. During the years, starting from the seminal work of 
\citep{dudley1969speed}, this setting has led to many works that have focused on establishing rates for the convergence of $\ProbHat$ to $\Prob$ as the number of samples increases \citep{bolley2007quantitative,canas2012learning, dereich2013constructive, boissard2014mean, fournier2015rate, weed2019sharp, lei2020convergence}. Although various choices are possible to quantify the distance between $\ProbHat$ and $\Prob$, the Wasserstein distance \citep{villani2008optimal, panaretos2019statistical} has arguably emerged as the standard choice: the Wasserstein distance is well defined when one of the distributions has discrete support and convergence in Wasserstein distance implies weak convergence \citep{villani2008optimal}. Moreover, compared to other distances based on optimal transport, such as the Sinkhorn divergence, the Wasserstein distance admits sharp dual characterizations that, in many important cases, yield convex reformulations of worst-case probabilities over Wasserstein balls \citep{mohajerin2018data}. This has led to various recent works that have derived sharp finite-sample convergence rates in Wasserstein distance for the closeness of $\ProbHat$ to $\Prob$ \citep{canas2012learning,fournier2015rate,valiant2016instance,singh2018minimax,weed2019sharp,lei2020convergence,niles2022estimation}. Unfortunately, the resulting convergence rates and confidence intervals still often require too many samples to be effectively used in practice, and often require constants that are difficult to compute. This has become a key limitation that prevents the application of these results in many applied fields, where only finite samples are commonly available and limiting the number of samples and the size of the support of the resulting $\ProbHat$, while guaranteeing a sufficiently small error threshold, has now become an essential prerequisite \citep{mohajerin2018data,shafieezadeh2015distributionally,boskos2023high,gracia2024data}. Closing this gap is the goal of this paper.

The conservativeness of the existing convergence rates is due to the fact that if $\Prob$ is fully unknown, then one must derive rates that are valid for all possible distributions and that hold in the arbitrarily small error regime. As a result, in general, if each sample has support in $\sX \subset \realNum^d$, a convergence rate of order $N^{-\frac{1}{d}}$, due to the discretization of $\sX,$ is unavoidable \citep{dudley1969speed,fournier2015rate,weed2019sharp}. Nevertheless, in practice, a multiscale behavior is often observed, where the error decreases faster when it is not small and the theoretical convergence rate is only approached close to the limit \citep{weed2019sharp}, when the error is often already smaller than what is acceptable in many practical applications.
Motivated by these results, in this paper, we focus on the finite sample setting and show how the information in the samples $\samplesVector$ can be exploited to automatically restrict the class of distributions to be considered and obtain substantially tighter high probability bounds on the distance between $\Prob$ and $\ProbHat$, without making additional, and generally restrictive, hypothesis, as common in the literature \citep{dudley1969speed,boissard2014mean,weed2019sharp,singh2018minimax,niles2022estimation}. The focus is on shifting from standard uniform convergence rates that have appeared in the literature to compute high probability bounds of the form $\wasserstein_{\rho}(\Prob,\ProbHat)\leq \epsilon(\samplesVector)$, where $\wasserstein_{\rho}(\Prob,\ProbHat)$ is the $\rho$-Wasserstein distance between  $\Prob$ and $\ProbHat$, and $\epsilon(\samplesVector)$ is an error bound that depends itself on the data. Critically, we show that if we let $\epsilon$ be a random variable dependent on the data, then, by solving a set of tractable optimization problems (mixed integer linear program (MILP)), one can obtain tighter error bounds compared to existing approaches. The rationale is that if one accepts having an error that is not known until data are collected, then substantially tighter bounds can be obtained by relying on the properties of the observed data.
We should stress that this trade-off is generally acceptable in many applications and is commonly exploited in many techniques from various fields, including statistics, optimization, and machine learning  \citep{vovk2005algorithmic,angelopoulos2021gentle,campi2018wait}.

Our approach relies on results from optimal transport, quantization theory, stochastic optimization, and unsupervised learning. We first show that if $\ProbHat$ has discrete support, then, even if $\Prob$ is unknown, $\wasserstein_{\rho}(\Prob,\ProbHat)$ 
can be upper bounded with high confidence by solving a MILP  whose number of variables only depends on the size of the support of $\ProbHat$. Then, we show that the resulting MILP can be further relaxed into a MILP with only $1$ integer variable independently of $\Prob$ and $\ProbHat$, thus guaranteeing efficiency. By relying on the relationship between $K$-means \citep{mcqueen1967some} and Wasserstein distance \citep{villani2008optimal}, we then present a modified version of the Lloyd's algorithm to find the $\ProbHat$ that approximately minimize $\wasserstein_{\rho}(\Prob,\ProbHat)$. On a series of benchmarks, both synthetic and taken from real-world datasets, we show that our framework substantially outperforms state-of-the-art by returning distributions of smaller support and tighter error bounds. 
For instance, on the MiniBooNE dataset \citep{miniboone_uci} with ambient dimension $d=50$, our method is able to approximate the unknown data distribution with a discrete distribution of support size of only $M=100$ yielding a high-confidence $\wasserstein_2$ bound of $0.5$, compared to state-of-the-art approaches based on the empirical distribution that produces a distribution of support size of over $10^5$ (the number of samples) yielding a bound for the same confidence on the $\wasserstein_2$ greater than  $11$.


The rest of the paper is organized as follows. Section \ref{section:preliminaries} introduces the mathematical foundations and notation used throughout the work. Section \ref{section:main-theoretical} states the main theoretical results, \changed{which are then used in Section \ref{section:partition} to devise our distribution-learning framework}. In Section \ref{section:experiments}, we evaluate the proposed framework on both synthetic and real-world datasets, and compare our approach with state-of-the-art. Proofs and additional details can be found in the Appendix. 


\section{Preliminaries}\label{section:preliminaries}
We start by introducing the mathematical notation used in this paper and giving a formal introduction to the $\rho$-Wasserstein distance.

\subsection{Notation}\label{prelim:notation}
For a vector $x \in \realNum^d$,  we denote by $x^{(i)}$ its $i$-element. The vector of ones is defined as $\bar\indicator=(1,\dots,1)^\top$. The set of indexes $\{1,\dots,N\}$ is denoted as $[N]$. For a given set $C \subseteq \sX$, we denote $\norm{C} = \sup_{x,y\in C}\norm{x-y}$ as the diameter of $C$. We represent $\indicator_{C}(x)= 
    1 \,\text{ if } x\in C; 
    0  \, \text{ otherwise}$ as the indicator function for the set $C$. For $\sX \subseteq \realNum^d$, $\big\{ C_1,...,C_M \big\}$ is a partition of $\sX$ in $M$ \emph{regions} if $C_i \subseteq \sX$, $\bigcup_{i=1}^M C_i = \sX$, and $C_i \cap C_j = \emptyset \; \forall i \neq j$.
Given a Borel measurable space $\sX \subseteq \realNum^d$, we denote by $\mathcal{B}(\sX)$ the Borel sigma algebra over $\sX$ and by $\sP(\sX)$ the set of probability distributions on $\sX$.
For a random variable $\vx$ taking values in $\sX$, $\vx \sim \Prob \in \sP(\sX)$ represents the probability measure associated to $x$. 
For $N \in \natNum$, $\sSimplex^N := \{ \omega \in \realNum^{N}_{\geq 0} \; : \; \sum_{i=1}^{N} \evomega{i} = 1 \}$ is the $N$-probability simplex. We denote with $\sSimplex_{[a,b]}^N := \{ \omega \in \realNum^{N}_{\geq 0} \; : \; \sum_{i=1}^{N} \evomega{i} = 1, \; a^{(i)} \leq \evomega{i} \leq b^{(i)} \}$ the box-constrained probability simplex. A discrete probability distribution $\ProbD \in \sP(\sX)$ is defined as $\ProbD=\sum_{i=1}^N\evpi{i}\delta_{\vc_i}$, where $\delta_\vc$ is the Dirac delta function centered at location $\vc \in \sX$ and  $\pi \in \sSimplex^N$, and $N$ is the number of locations in the support of $\ProbD$. Finally, given $N$ independent and identically distributed samples $\{x_i\}_{i=1}^{N}$ each defined in a probability space $(\sX,\sB(\sX),\Prob)$, the random samples vector $(x_1,\dots,x_N)\sim \Prob^N$, where $\Prob^N$ is the product probability of $\Prob$ $N$ times, and is defined in the probability space $(\sX^N,\sB(\sX)^N,\Prob^N)$.

\subsection{Wasserstein Distance}\label{prelim:wasserstein-distance}
Let  $\rho \geq 1$, and $\sX\subseteq \realNum^d$. For $m \in \natNum \cup \{ \infty \}$, $\norm{.}$ represent the $L_{m}$-norm. Further, define $\sP_{\rho}(\sX)$ as the set of probability distributions with finite $\rho$-th moments, i.e. all $\Prob \in \sP(\sX)$ are such that $\int_{\sX} \norm{x}^\rho d\Prob(x) < \infty$. Then, for $\Prob, \ProbQ \in \sP_{\rho}(\sX)$ the $\rho$-Wasserstein distance $\wasserstein_{\rho}$ between $\Prob$ and $\ProbQ$ is defined as
\begin{equation} \label{eq:wasserstein-distance-definition}
    \wasserstein_{\rho}(\Prob, \ProbQ) = \sT_\rho(\Prob, \ProbQ)^{\frac{1}{\rho}},
\end{equation}
where 
\begin{align}\label{eq:power-wasserstein-definition}
    \sT_\rho(\Prob, \ProbQ) = \inf_{\gamma \in \Gamma(\mathbb{P}, \mathbb{Q})} \int_{\sX \times \sX} \norm{x-y}^{\rho} d\gamma(x, y),
\end{align}
and $\Gamma(\Prob, \ProbQ) \subset \sP_{\rho}(\sX \times \sX)$ represents the set of joint probability distributions with given marginals $\mathbb{P}$ and $\mathbb{Q}$ (also known as \emph{couplings} between $\mathbb{P}$ and $\mathbb{Q}$), i.e., for all $\gamma \in \Gamma(\Prob, \ProbQ)$, 
\begin{align*}
    \gamma(A \times \sX) = \Prob(A), \; \gamma(\sX \times A) = \ProbQ(A) \qquad \forall A \in \sB(\sX).
\end{align*}
While computing $\wasserstein_{\rho}(\Prob, \ProbQ)$ is generally infeasible, when both distributions are discrete, $\wasserstein_{\rho}(\Prob, \ProbQ)$ can be computed as the solution of a finite linear program (LP) \citep{peyre2019computational}. In particular,  for $\Prob=\sum_{i=1}^M \evomega{i}\delta_{c_i}$ and $\ProbQ=\sum_{i=1}^M \evpi{i}\delta_{z_i}$, we have that
\begin{align}\label{eq:discrete-wasserstein-lp-optimization}
   \sT_{\rho}(\Prob, \ProbQ) &= \inf_{\gamma\geq 0} \sum_{i, j=1}^M \norm{c_i-z_j}^\rho \evgamma{i,j} 
   \\
    &
    \qquad 
    \text{s.t. }\;
    \left\{
    \begin{array}{l}
    \nonumber
    \gamma \bar\indicator = \omega \\
    \gamma^\top \bar\indicator = \pi
    \end{array}
    \right.
\end{align}
which is a linear program in $M^2$ variables.


\section{Theoretical Results}\label{section:main-theoretical}
We start by introducing the theoretical results enabling our framework to efficiently bound the Wasserstein distance between an unknown distribution and one built from data. Then, in Section \ref{section:partition}, we demonstrate how these results can be combined with intelligent clustering algorithms to develop an algorithmic framework for efficiently learning a distribution from data with bounds on the Wasserstein distance.

\subsection{A-Posteriori Probability Bounds}
We consider a vector of independent and identically distributed (i.i.d.) samples $\samplesVector=( \vx_1,...,\vx_N )$ of size $N$ from an unknown distribution $\Prob$ with support contained in a bounded set $\sX \subset \realNum^d$. Our goal is to employ the samples to approximate $\Prob$ with a discrete distribution $\ProbHat$ with error bounds on the $\rho$-Wasserstein distance. To do so, given a partition of $\sX$ into $M$ sets $\big\{ C_1,...,C_M \big\}$ and representative points $\{\vc_1,\hdots,\vc_M\}$ with $\vc_i\in C_i$, we define the following clusterized empirical distribution, where each of the $N$ samples is assigned to its corresponding region, that is, for $\pi^{(i)}= \frac{1}{N}\sum_{n=1}^{N}  \indicator_{C_i}(\vx_n)$ we have
\begin{align}
\label{eq:clusterized-empirical}
\ProbHat=\sum_{i=1}^M\pi^{(i)}  \delta_{c_i},
\end{align}
where $\delta_{c_i}$ is the Dirac delta distribution centered in $c_i$ and we call $\pi=(\pi^{(1)},...,\pi^{(M)})$. Intuitively, $\ProbHat$ is a generalization of the standard empirical distribution  \citep{fournier2015rate}, where samples are clustered per region.
  We should already stress that our results remain valid independently of $M$ and the choice of sets $\big\{ C_1,...,C_M \big\}$ and points $\{\vc_1,\hdots,\vc_M\}$. Nevertheless, in Section \ref{section:partition},  we will devise an algorithmic framework that relies on the results of this Section to automatically select regions $\{C_1,...,C_M\}$ and representative points $\{\vc_1,\hdots,\vc_M\}$ to minimize $ \wasserstein_{\rho}(\Prob, \ProbHat)$.

Existing results to bound $ \wasserstein_{\rho}(\Prob,\ProbHat)$ with high confidence rely on worst-case reasoning to derive uniform (tight) bounds that only depend on $N$, $d,$ and $\mathcal{X}$ (or alternatively on some assumptions on the moments of $\Prob$) \citep{fournier2015rate,weed2019sharp}. In what follows, 
we propose an alternative approach: we bound $ \wasserstein_{\rho}(\Prob,\ProbHat)$ with a high confidence that depends on the collected data. That is, we trade the fact that the confidence can only be computed after the data have been observed, to obtain tighter bounds. Our reasoning is based on the fact that
the error to quantize a distribution $\Prob$ into a discrete distribution is highly dependent on $\Prob$ itself \citep{dereich2013constructive}. For instance, a distribution close to a delta Dirac intuitively requires less data to be approximated with a certain error than a uniform distribution over a large support, and, while we cannot generally make this assumption a priori, we can rely on data to automatically rule out worst-case distributions that would need to instead be taken into account without additional information. Theorem \ref{thm:main-theorem} below formalizes this reasoning by combining recent advances in optimal transport and stochastic optimization with Clopper-Pearson confidence intervals built from data.
\begin{theorem}\label{thm:main-theorem}
 Given a confidence $\beta >0$ and a partition $\big\{ C_1,...,C_M \big\}$ with representative points $\{\vc_1,\hdots,\vc_M\}$, associate to each set $C_i$ parameters $\evlowerb{i},\evupperb{i}$, where $\evlowerb{i}=0$ if $\sum_{n=1}^{N}  \indicator_{C_i}(\vx_n)=0$, otherwise $\evlowerb{i}$ is such that
\begin{equation}\label{eq:clopper-pearson-lower-bound}
    \frac{\beta}{2M}=\sum_{j=\sum_{n=1}^{N}  \indicator_{C_i}(x_n)}^N \binom{N}{j}(\evlowerb{i})^j(1-\evlowerb{i})^{N-j},
\end{equation}
and $\evupperb{i}=1$ if $\sum_{n=1}^{N}  \indicator_{C_i}(x_n)=1$, otherwise $\evupperb{i}$ is such that
\begin{equation}\label{eq:clopper-pearson-upper-bound}
    \frac{\beta}{2M}=\sum_{j=0}^{\sum_{n=1}^{N}  \indicator_{C_i}(\vx_n)} \binom{N}{j}(\evupperb{i})^j(1-\evupperb{i})^{N-j}.
\end{equation}
Then, it holds that $\wasserstein_{\rho}(\Prob,\ProbHat) \leq \epsilon(\samplesVector)$ with probability at least $1-\beta$,
where $\epsilon(\samplesVector)$ is the solution of the following mixed integer linear problem (MILP):
\begin{align}\label{eq:def-epsilon}
\epsilon(\samplesVector)^{\rho}&= \sup_{\substack{
\omega \in \constrainedSimplex{\lowerb}{\upperb}^M, \\
\gamma \in \Delta^{M \times M}
}} \sum_{i=1}^M \sum_{j=1}^M \max_{x \in C_i} \norm{x-c_j}^{\rho} \evgamma{i,j} \\
    &\qquad \text{s.t. }\;
\left\{
\begin{array}{l}
\nonumber
\gamma \bar\indicator = \omega \\
\gamma^\top \bar\indicator = \pi \\
\evgamma{i,i} \geq \min \bigl\{ \evomega{i}, \evpi{i} \bigr\},
\quad \forall i \in [M]
\end{array}
\right.
\end{align}
\end{theorem}
In Theorem \ref{thm:main-theorem}, we first use the data to restrict with confidence $1-\beta$ to the distributions that assign mass within $ [\evlowerb{i}, \evupperb{i}]$ to $C_i$ for any $i\in\{1,...,M\}$. Then, within this set of distributions, we seek the one that maximizes the Wasserstein distance from $\ProbHat$. Unfortunately, as $\Prob$ is unknown, computing the optimal transport plan that would lead to true Wasserstein distance is infeasible. Furthermore, even if $\Prob$ were known, in general, computing exactly $\wasserstein_{\rho}(\Prob,\ProbHat)$ would lead to an infinite-dimensional optimization problem. To obtain a tractable problem (a MILP) that can be efficiently solved with existing tools, e.g., Gurobi \citep{gurobi}, instead, we relax the transport plan to a sub-optimal one that is constrained to keep as much mass as possible in the same region (i.e., condition $\evgamma{i,i} \geq \min\{ \evomega{i}, \evpi{i} \}$), where the transport cost is minimized.  We should also stress that the confidence in Theorem \ref{thm:main-theorem} is with respect to  $\Prob^{N}$, the product measure of $\Prob$ $N$ times, that is, the probability induced by the sample vector $\samplesVector$. A detailed proof can be found in Section \ref{section:proofs} and for further details, as well as an explicit proof of the measurability of the event $\{ \wasserstein_{\rho}(\Prob,\ProbHat) \leq \epsilon(\samplesVector) \}$, we refer the interested reader to Section \ref{proof:measurability}. 

Note that, as detailed in the proof of Proposition \ref{prop:wasserstein-bound-no-confidence} in the Appendix, the MILP in \eqref{eq:def-epsilon} arises as a relaxation of maximin linear program and by fixing a transport plan that is optimal only in the diagonal elements (i.e., in the $\evgamma{i,i}$). In practice, one could try to fix a transport plan also for the non-diagonal elements of $\gamma,$ instead than considering a worst-case approach as in Theorem \ref{thm:main-theorem}. However, because of the supremum over $\omega$, this would generally require non-linear normalization coefficients that would greatly increase the computational complexity of the resulting problem. Another alternative could be to restrict $\{C_1,...,C_M\}$ to be a dyadic partition and then rely on existing results for this setting (e.g., Proposition 1 in \citet{weed2019sharp}). However, as we will discuss in more detail in Section \ref{section:conclusion}, this restriction would limit the flexibility of the resulting framework.

\begin{remark}[Comparison with existing literature]\label{rmk:comparison-existing-literature}
We would like to remark that Theorem \ref{thm:main-theorem} should not be considered as a replacement for the convergence rates provided in \citep{fournier2015rate} or \citep{weed2019sharp}, which have been shown to be tight. In fact, our bound serves a different purpose: by providing error bounds that are a function of the data, Theorem \ref{thm:main-theorem} aims to obtain tighter non-asymptotic bounds for a specific given problem. This is substantially different from the approach and problem taken in \citep{fournier2015rate,weed2019sharp}, where the resulting convergence rates are valid for any distribution with a specific support size. Consequently, it is natural that by focusing on any possible distribution, the resulting bounds (on average) will be worse than those in Theorem \ref{thm:main-theorem}, where the information given by the specific data is exploited to narrow the set of distributions that must be considered in the bounds by solving a set of optimization problems. In Section \ref{section:experiments}, we will present an empirical comparison of these results with our framework.
\end{remark}

\subsection{Computational Complexity and Further Relaxations}\label{rmk:computational-efficiency}
At first glance, the reader may find Theorem \ref{thm:main-theorem} computationally expensive due to the need to solve a MILP to compute $\epsilon(\samplesVector)$. However, we would like to remark that the MILP in \eqref{eq:def-epsilon} has a complexity that only depends on $M$, the support of $\ProbHat$, which is generally much smaller than $N$. In practice, even if solving a MILP has a worst case theoretical complexity that is exponential in the number of integer variables, due to the presence of various heuristics for this class of problems, solving \eqref{eq:def-epsilon} using state-of-the-art tools such as Gurobi \citep{gurobi} and a common machine can be generally performed in a few seconds for $M<200$ (a value of $M$ that generally suffices for obtaining relatively small errors with our framework, as we will show in Section \ref{section:experiments}). 
To further improve scalability and remove the worst case exponential complexity in $M$, in Proposition
\ref{prop:milp-into-lp} below we show that \eqref{eq:def-epsilon} can be relaxed into a MILP with only one integer independently of $M$ and $N$.
\begin{proposition}\label{prop:milp-into-lp}  
Consider the setting in Theorem \ref{thm:main-theorem}. Let $v = \arg\min_{z \in \text{Vert}\big( \constrainedSimplex{\lowerb}{\upperb}^M \big)} \norm{z - \pi}$ be the closest vertex of set $\constrainedSimplex{\lowerb}{\upperb}^M$ to $\pi$.
    Denote $I, J \subset [M]$ as the sets of indexes such that $\evv{i}=\evlowerb{i}$ for $i\in I$, and $\evv{j}=\evupperb{j}$ for $j\in J$. Then, it holds that
    \begin{align}\label{eq:eps-in-lp}
        \wasserstein_{\rho}(\Prob,\ProbHat) &\leq \wasserstein_\rho(\sum_{i=1}^M \evv{i}\delta_{c_i}, \ProbHat) + \xi(\samplesVector),
    \end{align}
    with probability at least $1-\beta$, where $\xi(\samplesVector)$ is the solution to the following MILP with one integer variable:
    \begin{align}\label{eq:def-xi}
\xi(\samplesVector)^\rho &=\sup_{\substack{
\omega \in \constrainedSimplex{\lowerb}{\upperb}^M, \\
\gamma \in \Delta^{M \times M}
}} \sum_{i=1}^M \sum_{j=1}^M \max_{x \in C_i} \norm{x-c_j}^{\rho} \evgamma{i,j} \\
&\qquad \text{s.t. }\;
\left\{
\begin{array}{l}
\nonumber
\gamma \bar\indicator = \omega \\
\gamma^\top \bar\indicator = v \\
\evgamma{i,i} \geq \evv{i}, \quad \forall i \in I \\
\evgamma{j,j} \geq \evomega{j}, \quad \forall j \in J \\
\evgamma{f,f} \geq \min \bigl\{ \evomega{f}, \evv{f} \bigr\} \text{ for the free index } f = [M] \setminus (I \cup J)
\end{array}
\right.
\end{align}
\end{proposition}
A detailed proof of Proposition \ref{prop:milp-into-lp} can be found in Section \ref{section:proofs}, where we rely on the triangle inequality with the distribution $\sum_{i=1}^M \evv{i}\delta_{c_i}$ to be able to reduce the number of integer variables in the resulting optimization problem for $\xi(\samplesVector)$, while limiting the conservativeness introduced. In fact, Proposition \ref{prop:milp-into-lp} allows one to relax the MILP in \eqref{eq:def-epsilon} into a LP with $M^2$ variables (needed to compute $\wasserstein_\rho(\sum_{i=1}^M \evv{i}\delta_{c_i}, \ProbHat)$, see Section \ref{prelim:wasserstein-distance}) and an MILP with $1$ integer variable and $M^2+M$ continuous variables. As we will show in Section \ref{section:experiments}, this has the effect of greatly improve the computational efficiency at the cost of slightly more conservative bounds.


\begin{remark}[Finding the vertex $v$] Proposition \ref{prop:milp-into-lp} is stated for a specific $v \in \text{Vert}\big( \constrainedSimplex{\lowerb}{\upperb}^M \big)$, the one that minimizes the distance to $\pi$. However, we should stress that Proposition \ref{prop:milp-into-lp} would hold for any other vertex of $\constrainedSimplex{\lowerb}{\upperb}^M$. Yet, choosing  $v = \arg\min_{z \in \text{Vert}\big( \constrainedSimplex{\lowerb}{\upperb}^M \big)} \norm{z - \pi}$ yields a good performance in practice. In fact, not only having $v$ close to $\pi$ helps reducing the conservativeness of the problem, but finding $v = \arg\min_{z \in \text{Vert}\big( \constrainedSimplex{\lowerb}{\upperb}^M \big)} \norm{z - \pi}$ can be done particularly efficiently via a sorting algorithm. In fact, as minimizing $\norm{z-\pi}$ is equivalent to minimizing $\norm{z-\pi}^m = \sum_{i=1}^M |\evz{i}-\evpi{i}|^m$, we can first start choosing $\evv{i}=\evlowerb{i}$, sort indexes by the gap $|\evpi{i}-\evupperb{i}|^m-|\evpi{i}-\evlowerb{i}|^m$ ascendingly, and then progressively increase $\evv{i}$ (capping at $\evlowerb{i}$) for the sorted indexes until the condition $\sum_{i=1}^M \evv{i}=1$ is met. Intuitively, this algorithm fills up the probability mass in the coordinates where the distance to the upper bound is smaller than to the lower bound.
\end{remark}

\subsection{Convergence Guarantees}\label{subsection:convergence-guarantees}
Even if our bounds are specifically designed for non-asymptotic analysis, before introducing an algorithmic framework to automatically design regions $\{C_1,...,C_M\}$, we need to discuss the behavior of the bounds in Theorem \ref{thm:main-theorem} and Proposition \ref{prop:milp-into-lp} when the number of samples increases. In fact, while, for a fixed confidence, increasing $N$ has the effect of shrinking all intervals $[\evlowerb{i}, \evupperb{i}]$ (see \eqref{eq:clopper-pearson-lower-bound} and \eqref{eq:clopper-pearson-upper-bound}), increasing $M$ on one side increases the support of $\ProbHat$ $-$ reducing the quantization error, but simultaneously linearly increases the number of events that must be considered in the probability computations (see Appendix \ref{section:proofs}). In Proposition \ref{prop:convergence-bound-n-m} below, we show that, despite this tension, under standard consistency assumptions in the partition, if $M=N^{\alpha}$ for any $\alpha\in (0,1)$, then the bounds in Theorem \ref{thm:main-theorem} and Proposition \ref{prop:milp-into-lp} always converge to $0$ increasing $N$ and $M$.
\begin{proposition}\label{prop:convergence-bound-n-m}
   Given an error threshold $\delta \in (0, 1)$, for $\alpha \in (0,1)$ assume that $M= N^\alpha$ and that the partition $\{ C_i \}_{i=1}^M$ is such that $\ProbHat(\cup_{i=1}^{M-1}C_i) \geq 1-\frac{\delta^\rho}{2\norm{\sX}^\rho}$ and   $C_M=\sX \setminus \cup_{i=1}^{M-1}C_i$. Further, assume that $\max_{i \in [M-1]}\norm{C_i} \leq  \frac{L}{M}$ for a constant $L>0$. Then, there exists $N^* \in \natNum$ such that for any $N \geq N^*$, it holds that 
   $\epsilon(\samplesVector)\leq \delta$ and $\wasserstein_\rho(\sum_{i=1}^M \evv{i}\delta_{c_i}, \ProbHat) + \xi(\samplesVector)\leq \delta$, where $\epsilon(\samplesVector)$ and $\xi(\samplesVector)$ are as defined in Theorem \ref{thm:main-theorem} and Proposition \ref{prop:milp-into-lp}, respectively.
\end{proposition}
\begin{figure}[htbp]
    \centering
    \begin{subfigure}[t]{0.24\textwidth}
        \centering
        \includegraphics[width=\linewidth]{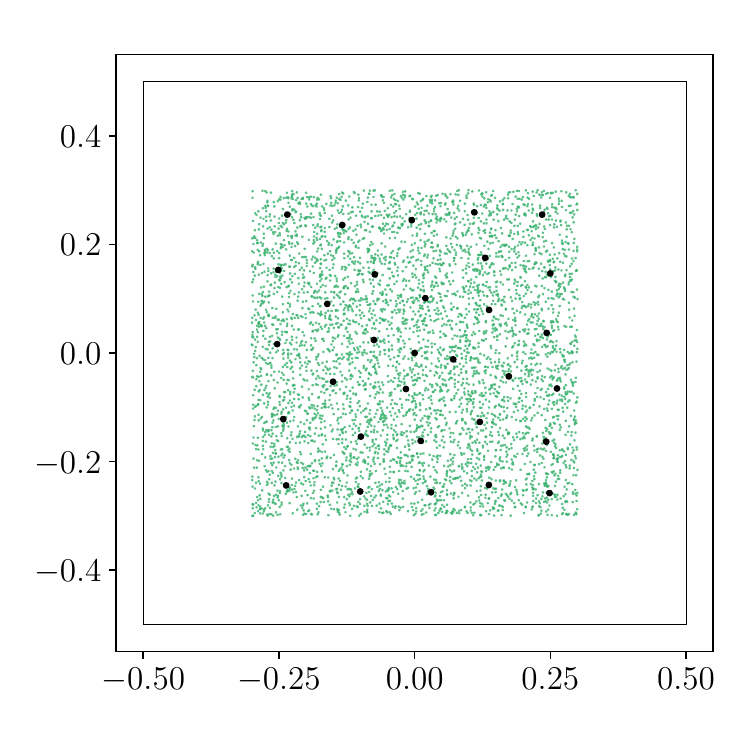}
        \caption{Clustering}
        \label{subfig:step1}
    \end{subfigure}
    \begin{subfigure}[t]{0.24\textwidth}
        \centering
        \includegraphics[width=\linewidth]{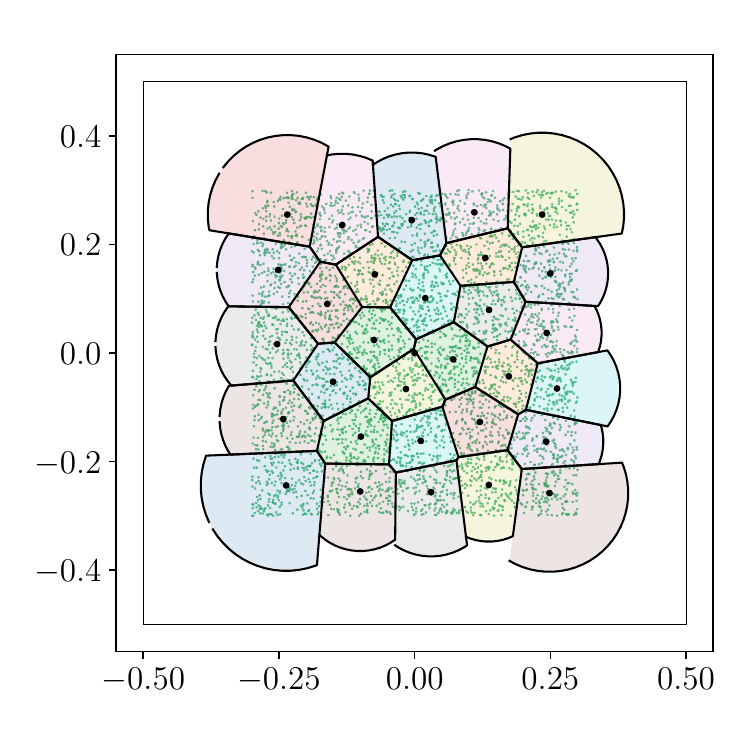}
        \caption{Partitioning}
        \label{subfig:step2}
    \end{subfigure}
    \begin{subfigure}[t]{0.24\textwidth}
        \centering
        \includegraphics[width=\linewidth]{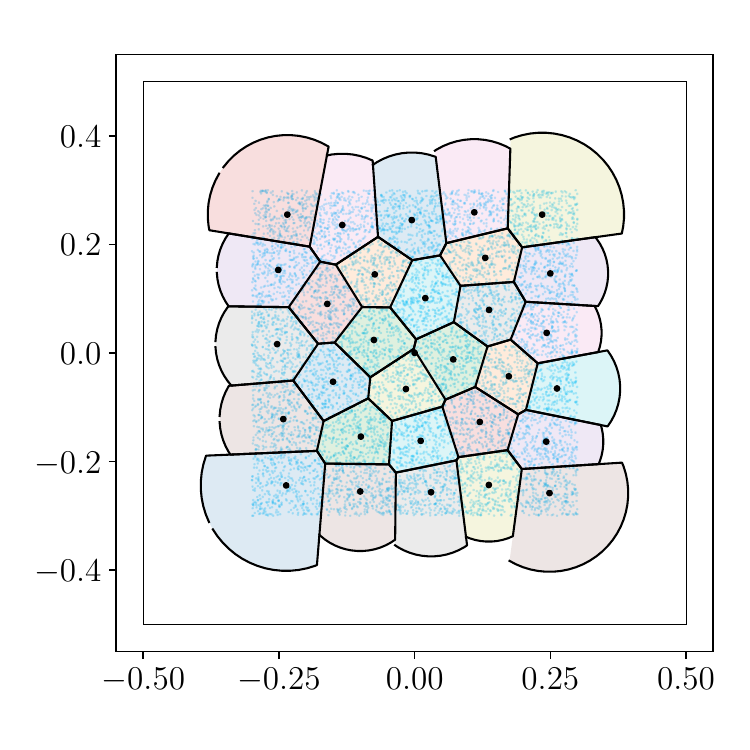}
        \caption{Counting Samples}
        \label{subfig:step3}
    \end{subfigure}
    \begin{subfigure}[t]{0.24\textwidth}
        \centering
        \includegraphics[width=\linewidth]{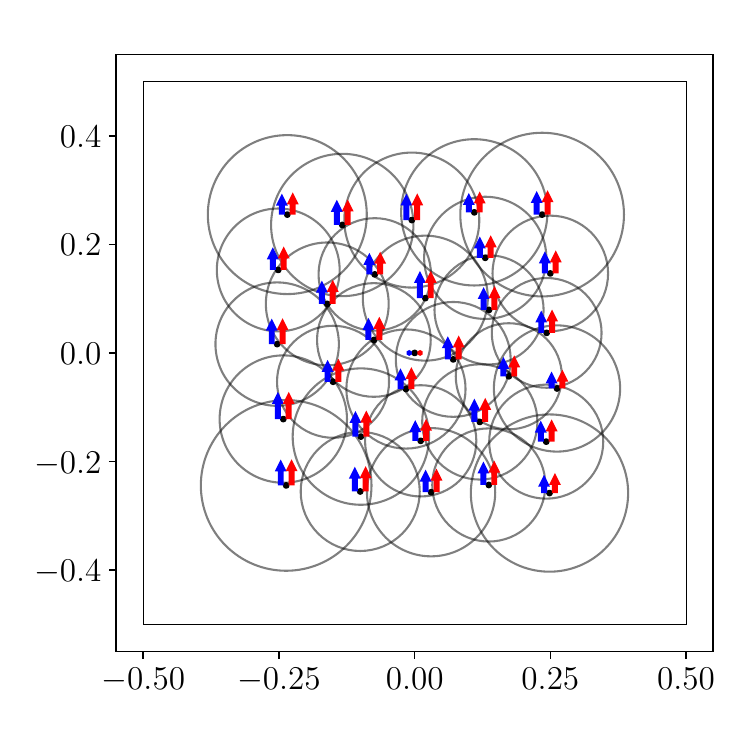}
        \caption{
        Confidence
        }
        \label{subfig:step4}
    \end{subfigure}
    \caption{
    Illustration of the data-driven construction of the discrete approximation $\ProbHat$ as defined in \eqref{eq:clusterized-empirical} with $M=30$ for an unknown 2D uniform distribution $\Prob$ with true support $[0.3, 0.3]^2$ and assumed support $\sX=[0.5, 0.5]^2$ (black box), using two independent datasets $\sD_{N_{\mathrm{train}}}$ ($N_{\mathrm{train}}=5 \times 10^3$, green dots) and $\sD_N$ ($N=10^4$, blue dots).
    (\ref{subfig:step1}) $(M-1)$-means clustering is applied to $\sD_{N_{\mathrm{train}}}$ to obtain representative points $\{c_i\}_{i=1}^{M-1}$ (black dots).
    (\ref{subfig:step2}) Radii $\{r_i\}$ are computed based on $\sD_{N_{\mathrm{train}}}$ as in Algorithm~\ref{alg:approximate-construction}, lines 3-5, defining regions $\{C_i\}_{i=1}^{M-1}$ (colored areas). 
    (\ref{subfig:step3}) The probability vector $\vpi$ is computed by counting, for each region, the coverage of dataset $\sD_N$ (blue dots).
    (\ref{subfig:step4}) Clopper-Pearson confidence intervals $[\evlowerb{i}, \evupperb{i}]$ for the probability of each region $i$ are computed according to \eqref{eq:clopper-pearson-lower-bound} and~\eqref{eq:clopper-pearson-upper-bound} (blue and red arrows). 
    The regions are contained within $L_m$-balls centered at $\vc_i$ with radii $r_i$ (grey circles).
    }
    \label{fig:approximate-construction}
\end{figure}

\section{\changed{Partition Construction}}\label{section:partition}
We now describe our approach to select the partition $\{C_1,\ldots,C_M\}$ and representative points $\{c_1,\ldots,c_M\}$ to minimize $\wasserstein_\rho(\Prob,\ProbHat)$. To do so, we make use of a held out training dataset $\sD_{N_{\mathrm{train}}} = (\vx_1,\ldots,\vx_{N_{\mathrm{train}}})$ different and independent from $\sD_N$\footnote{The use of a held out dataset $\sD_{N_{\mathrm{train}}}$ guarantees that $\{C_1,\ldots,C_M\}$ and $\{c_1,\ldots,c_M\}$ are independent of $\sD_N$, which is currently needed to rely on the Clopper-Pearson confidence intervals in our framework. }, and use the approach illustrated in Figure~\ref{fig:approximate-construction} and summarized in Algorithm~\ref{alg:approximate-construction}.
Specifically, as $\Prob$ is unknown, we select $\{C_1,\ldots,C_{M-1}\}$ and representative points $\{c_1,\ldots,c_{M-1}\}$ to minimize the Wasserstein distance between the empirical distribution induced by $\sD_{N_{\mathrm{train}}}$ and its clusterization, namely
\begin{equation}\label{eq:objective_partition}
    \wasserstein_\rho\left(\frac{1}{N_{\mathrm{train}}}\sum_{i=1}^{N_{\mathrm{train}}}\delta_{\vx_i}, \sum_{i=1}^{M-1}\elem{\tilde\vpi}{i}\delta_{\vc_i}\right),
\end{equation}
where $\elem{\tilde\vpi}{i}=\tfrac{1}{N_{\mathrm{train}}}\sum_{n=1}^{N_{\mathrm{train}}}\indicator_{C_i}(\vx_n)$.
The advantage of considering the objective in \eqref{eq:objective_partition} is that for a fixed set of representative points $\{c_1,\ldots,c_{M-1}\}$, the partition $\{C_1,\ldots,C_{M-1}\}$ minimizing \eqref{eq:objective_partition} is the $L_m$-norm Voronoi partition induced by these points on the (unknown) support of $\Prob$ (see, e.g., Lemma~3.1 in \citet{canas2012learning}). 
As shown in Eqn~(2) in \citet{canas2012learning}, the $\{c_1,\ldots,c_{M-1}\}$ minimizing \eqref{eq:objective_partition} are then given by the centroids of an optimal $(M-1)$-means clustering of $\sD_{N_{\mathrm{train}}}$. 
We therefore use Lloyd’s algorithm \citep{lloyd1982least} in line~1 of Algorithm~\ref{alg:approximate-construction} to find $\{c_1,\ldots,c_{M-1}\}$. 
The corresponding optimal regions $\{C_1,\ldots,C_{M-1}\}$ would be given by the $L_m$-norm Voronoi partition on the support of $\Prob$, which, however, is unknown and only assumed to be contained in $\sX$. 
Using data in $\sD_{N_{\mathrm{train}}}$, we therefore approximate the support of $\Prob$ as a union of $L_m$-balls centered at the representative points, i.e., we define
\[
    \sX_{\mathrm{approx}}=\bigcup_{i=1}^{M-1} \{\vx\in\realNum^d  \; : \;  \|\vx-\vc_i\|\leq r_i\},
\]
where the radii $r_i$ are determined as described in lines 3-5 of Algorithm~\ref{alg:approximate-construction}. In particular, first, $r_i$ is initialized as the maximum distance between the representative point $\vc_i$ and the samples in $\sD_{N_\mathrm{train}}$ assigned to $\vc_i$ by the $(M-1)$-means clustering. 
Then, to account for the fact that the balls of radius $r_i$ centered at the training samples may not cover the support of $\Prob$, the radius is increased to be at least half the distance to representative points associated with adjacent regions. Since computing the exact neighboring regions is computationally demanding for large $M$ and $d$, this neighborhood is approximated by the $\lceil 0.05\,M\rceil$-th nearest representative point, a heuristic that is empirically observed to provide good coverage (see Appendix~\ref{appendix:partition-analysis} for an empirical evaluation).
Defining the regions $\{C_1,\ldots,C_{M-1}\}$ as the Voronoi partition on the approximate support $\sX_{\mathrm{approx}}$ then yields
\begin{equation}\label{eq:approximate_voronoi_regions}
    C_i = \big\{\vx\in\realNum^d  \; : \;  \|\vx-\vc_i\| \leq \min\{\|\vx-\vc_j\|, r_i\}, \forall j\neq i \big\}.
\end{equation}
Finally, since $\sX_{\mathrm{approx}}$ is only an estimate of the true support and may be strictly contained in the assumed support $\sX$, we include the uncovered region $\sX \setminus \sX_{\mathrm{approx}}$ as an additional region $C_M$, with a representative point chosen as the center of $\sX$ (Algorithm~\ref{alg:approximate-construction}, line~7).
This ensure that the resulting collection $\{C_1,\hdots,C_{M-1}, C_M\}$ forms a partition of the assumed support $\sX$. 
An illustration of the resulting partition and representative points for different values of $M$ for a 2D Gaussian Mixture is shown in Figure~\ref{fig:example_quantizations}. 
Once $\{C_1,\ldots,C_M\}$ and $\{c_1,\ldots,c_M\}$ are selected, in line 6 of Algorithm \ref{alg:approximate-construction}, for $i\in\{1,...,M-1\}$, we   compute the weights $\evpi{i}$ defining $\ProbHat$ by assigning each sample in $\sD_n$ to the closest point $\vc_i$ within a radius $r_i$. Any point not assigned to any region is then assigned to $\vc_M$ (line 7).
For an an extensive quantitative analysis of the performance of Algorithm~\ref{alg:approximate-construction}, we refer the reader to Appendix~\ref{appendix:partition-analysis}.

\begin{remark}
    Note that the regions $\{C_1,\ldots,C_M\}$ are defined  by the points $\{\vc_1,\hdots,\vc_M\}$ and the associated radii $\{r_1,\hdots,r_M\}$, as defined in Algorithm \ref{alg:approximate-construction}. The quantities $r_i$ can also be used to over-approximate cost terms $\max_{x\in C_i} \|x-c_j\|$ appearing in~\eqref{eq:def-epsilon} of Theorem~\ref{thm:main-theorem}  (and Proposition~\ref{prop:milp-into-lp}). In practice, as we discuss in Appendix~\ref{appendix:partition-analysis}, this can lead to increased computational efficiency, especially for large $M$ and high dimensions $d$. 
\end{remark}

\begin{algorithm}[htbp]
\DontPrintSemicolon
\SetKwInOut{Input}{input}\SetKwInOut{Output}{output}
\caption{Data-driven construction of a discrete approximation $\ProbHat$ to $\Prob$}\label{alg:approximate-construction}
\Input{Set $\sX\subset\realNum^d$, i.i.d. samples $\sD_{N_\mathrm{train}}$, $\sD_N$ from $\Prob$, support size $M$ of $\ProbHat$.}
\Output{Discrete approximation $\ProbHat=\sum_{i=1}^{M}\evpi{i}\delta_{\vc_i}$}
Apply $(M-1)$-means clustering (Lloyd’s algorithm) to $\sD_{N_\mathrm{train}}$ and obtain centroids $\{\vc_i\}_{i=1}^{M-1}$\;
\For{$i = 1,\ldots,M-1$}{
    Set $r_i = \max\{\|\vx-\vc_i\| \mid \vx\in\sD_{N_\mathrm{train}} \text{ assigned to } \vc_i\}$\;
    Let $d_i$ be the distance from $\vc_i$ to its $\lceil 0.05\,M\rceil$-th nearest neighbor in $\{\vc_j\}_{j\neq i}^{M-1}$\;
    Update $r_i \leftarrow \max\{r_i,\tfrac{1}{2}d_i\}$\;
    Set $\evpi{i} = \frac{1}{N}\sum_{\vx\in\sD_N}
    \mathbbm{1}\!\left\{
    \|\vx-\vc_i\|\le r_i \;\text{and}\;
    \|\vx-\vc_i\|\le \|\vx-\vc_j\| \;\forall j\neq i
    \right\}$\;
}
Set $\evpi{M} \leftarrow 1-\sum_{i=1}^{M-1}\evpi{i}$ and $\vc_M$ to the (Chebyshev) center of $\sX$\;
\end{algorithm}

\begin{figure}[htbp]
    \centering
    \begin{subfigure}[t]{0.32\textwidth}
        \centering
        \includegraphics[width=\linewidth]{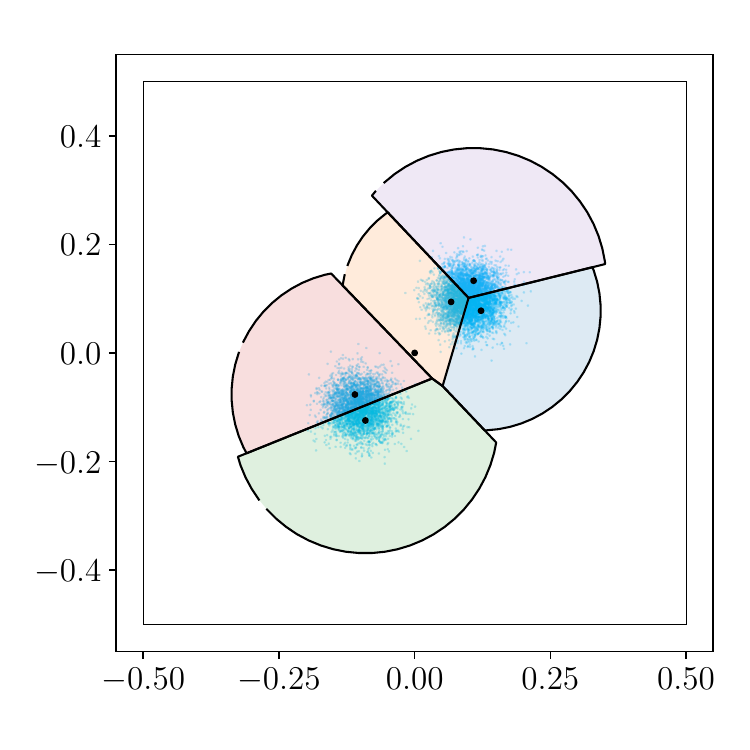}
        \caption{$M=5$}
    \end{subfigure}
    \begin{subfigure}[t]{0.32\textwidth}
        \centering
        \includegraphics[width=\linewidth]{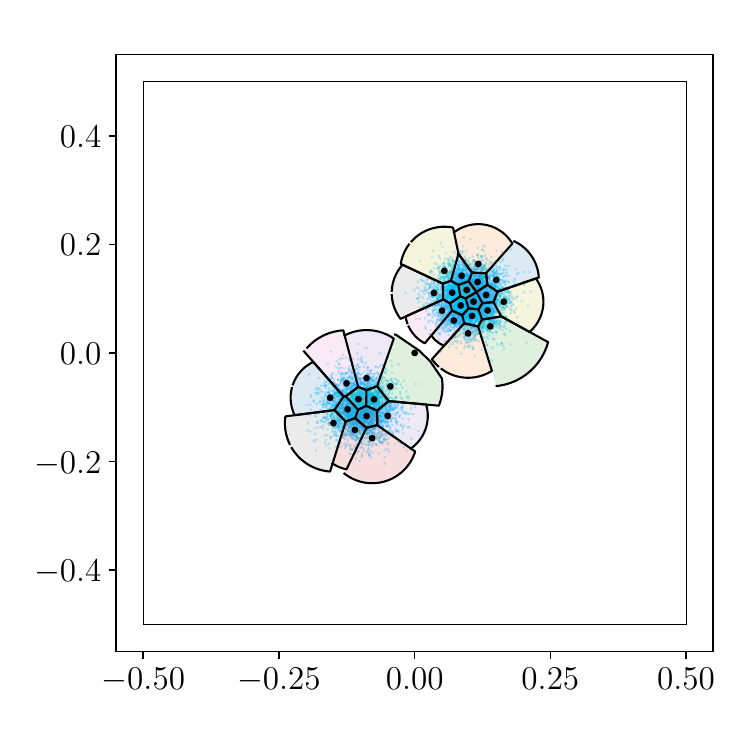}
        \caption{$M=30$}
    \end{subfigure}
    \begin{subfigure}[t]{0.32\textwidth}
        \centering
        \includegraphics[width=\linewidth]{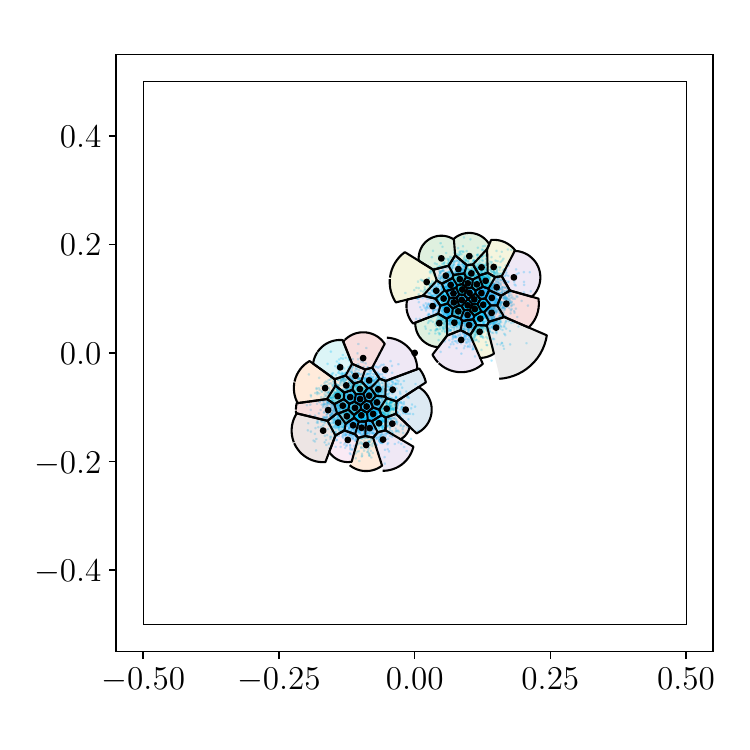}
        \caption{$M=75$}
    \end{subfigure}
    \caption{
    Illustration of the resulting representation points $\{\vc_i\}_{i=1}^M$ (black dots) and partition $\{C_i\}_{i=1}^M$ for different values of $M$ using $N_{\mathrm{train}}=5\times 10^3$ samples from an unknown bimodal 2D Gaussian mixture distribution with support truncated to $\sX=[-0.5,0.5]^2$. The independent dataset $\sD_N$ with $N=10^4$ (blue dots) is shown for validation of the resulting partition.
    }
    \label{fig:example_quantizations}
\end{figure}

\begin{figure}[htbp]
    \centering
    \begin{subfigure}[t]{0.32\textwidth}
        \centering
        \includegraphics[width=\linewidth]{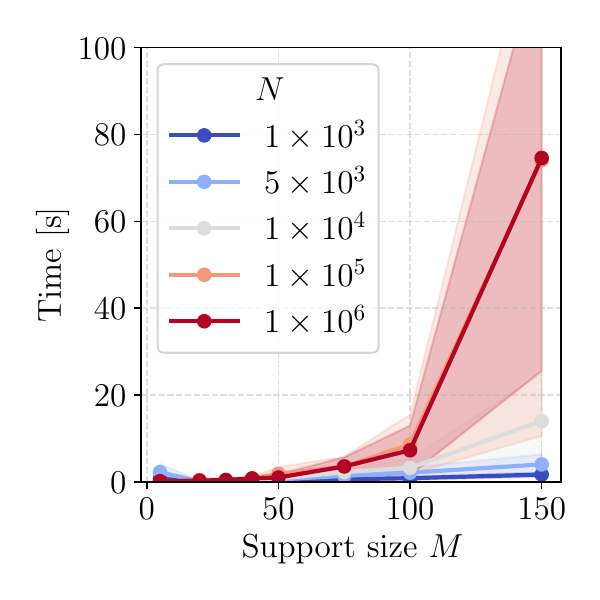}
        \caption{2D - Using Theorem \ref{thm:main-theorem} 
        }
        \label{fig:computation-times-thm-small-2D}
    \end{subfigure}
    \begin{subfigure}[t]{0.32\textwidth}
        \centering
        \includegraphics[width=\linewidth]{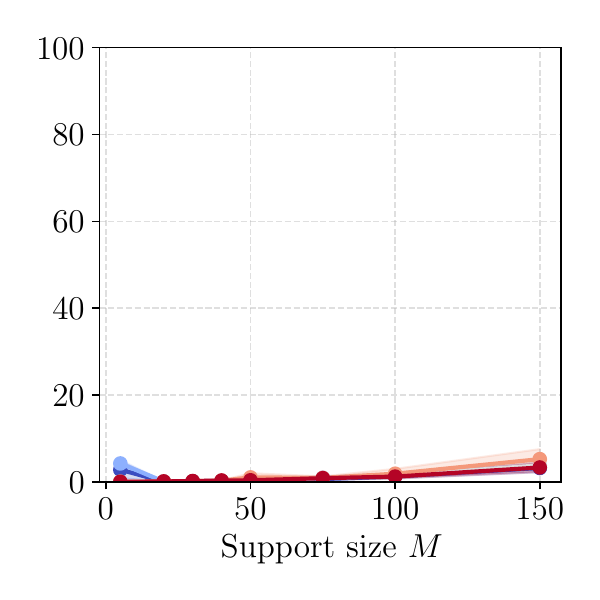}
        \caption{2D - Using Proposition \ref{prop:milp-into-lp}}
        \label{fig:computation-times-prop-small-2D}
    \end{subfigure}
    \begin{subfigure}[t]{0.32\textwidth}
        \centering
        \includegraphics[width=\linewidth]{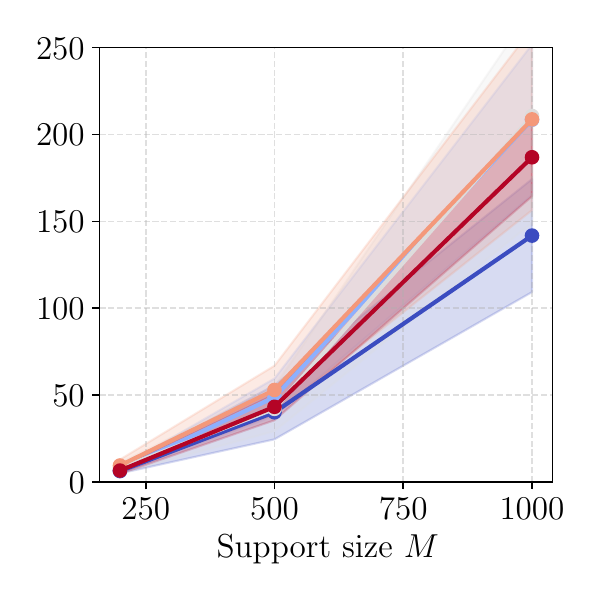}
        \caption{2D - Using Proposition \ref{prop:milp-into-lp}}
        \label{fig:computation-times-prop-large-2D}
    \end{subfigure}
    \begin{subfigure}[t]{0.32\textwidth}
        \centering
        \includegraphics[width=\linewidth]{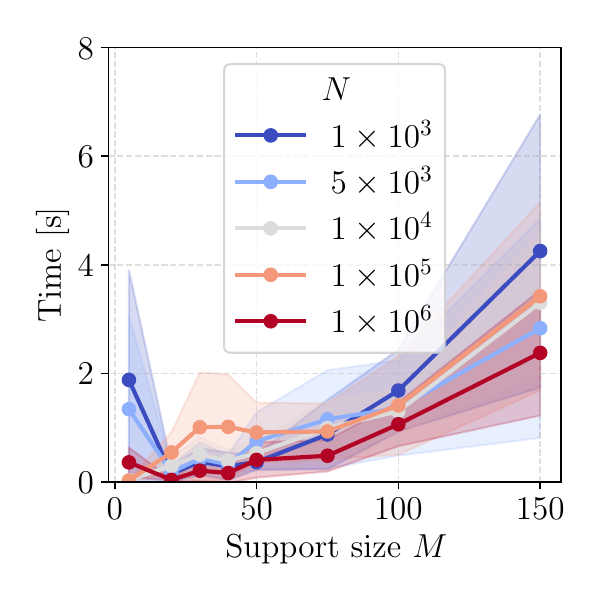}
        \caption{100D - Using Theorem \ref{thm:main-theorem} 
        }
        \label{fig:computation-times-thm-small-100D}
    \end{subfigure}
    \begin{subfigure}[t]{0.32\textwidth}
        \centering
        \includegraphics[width=\linewidth]{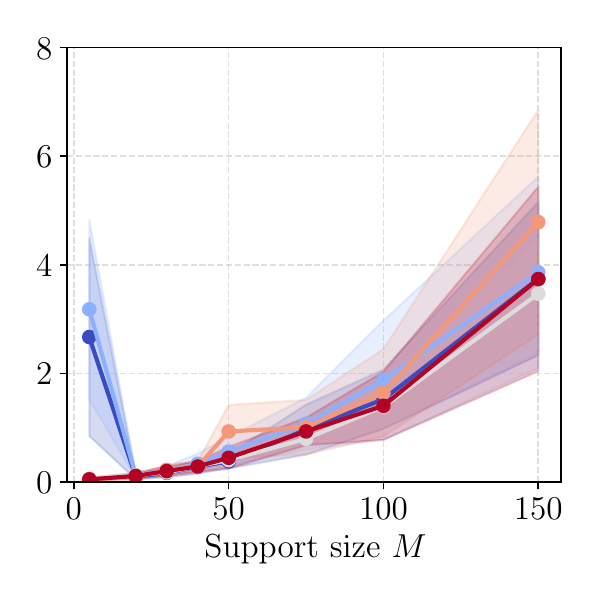}
        \caption{100D - Using Proposition \ref{prop:milp-into-lp}}
        \label{fig:computation-times-prop-small-100D}
    \end{subfigure}
    \begin{subfigure}[t]{0.32\textwidth}
        \centering
        \includegraphics[width=\linewidth]{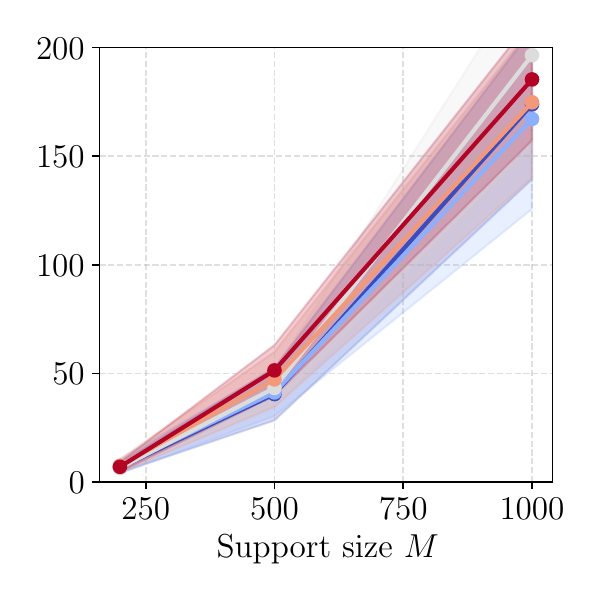}
        \caption{100D - Using Proposition \ref{prop:milp-into-lp}}
        \label{fig:computation-times-prop-large-100D}
    \end{subfigure}
    \caption{The computation times as the support size $M$ increases for multiple number of samples $N$ in the 2D and 100D Gaussian setting (for $\rho=2$).
    }
    \label{fig:computation-time}
\end{figure}

\section{Experimental Results}\label{section:experiments}
We empirically evaluate the performance of our distribution-learning approach by first comparing the bound in Proposition~\ref{prop:milp-into-lp} with  Theorem~\ref{thm:main-theorem}. 
We then benchmark the proposed bounds against state-of-the-art Wasserstein guarantees on both synthetic and real-world datasets, highlighting how the properties of the underlying distribution affect performance. 
\changed{
Our synthetic experiments consider truncated isotropic Gaussian distributions with dimensions ranging from $d=2$ to $d=100$ and varying variances, as well as uniform distributions with varying support sizes over the same range of dimensions. Real-world benchmarks include two UCI regression datasets \citep{dos2024co,miniboone_uci} and the safety-critical MedMNIST dataset \citep{medmnistv2}.

For all experiments, the support $\sX$ is taken as the $L_\infty$-ball of diameter $1$ centered at the origin. The partition is constructed using an independent training set of size $N_{\text{train}} \in \{1000, 5000\}$, and all confidence bounds are computed with confidence level $\beta = 10^{-6}$. The $\rho$-Wasserstein distance is defined with respect to the Euclidean norm (i.e., $m=2$ in~\eqref{eq:wasserstein-distance-definition}). 
Each experiments is repeated ten times with different random seeds, and results are reported as means with shaded bands indicating one standard deviation. 
Computations are performed on a HPC node with 8 CPU cores and 32\,GB RAM (Intel Xeon Gold 6148 @ 2.40\,GHz). 
The code used to produce the results in this work is available at \url{https://github.com/EduardoFMDCosta/WassersteinDistributionLearning}.
}

\subsection{MILP vs LP Relaxation}\label{subsection-experiment:milp-vs-lp-results}
To improve the scalability of our framework, we introduced Proposition~\ref{prop:milp-into-lp} as a relaxation of Theorem~\ref{thm:main-theorem}, replacing the MILP in \eqref{eq:def-epsilon} with a formulation that contains only a single integer variable. Therefore,  we compare the two formulations in terms of the computational scalability and tightness of the resulting bounds; in this subsection  we focus on computational complexity,  while the tightness of the bounds is studied in Section \ref{subsection-experiment:comparing-fournier}.
Figure~\ref{fig:computation-time} reports the computation time for varying support sizes $M$ of $\ProbHat$ and increasing numbers of samples $N$, for both a 2D and 100D Gaussian distributions, using at most eight CPU cores and 32~GB RAM\footnote{A similar trend is observed also for all the other benchmarks considered in this work.}. 
From Figure~\ref{fig:computation-time}, we observe that for small support sizes of $\ProbHat$, both formulations exhibit comparable performance. 
The difference in computational scaling becomes evident at larger support sizes ($M\geq100$ for 2D and $M \geq 500$ for 100D). 
In this regime, the computation time associated with Proposition~\ref{prop:milp-into-lp} increases with~$M$ but remains within a few minutes for the largest support sizes considered (see Figures~\ref{fig:computation-times-prop-large-2D} and \ref{fig:computation-times-prop-large-100D}), whereas the MILP formulation of Theorem~\ref{thm:main-theorem} frequently exceeds the $10$-minute timeout limit for most random seeds. 
Importantly, for Proposition~\ref{prop:milp-into-lp}, and for any fixed $M$, the computation time remains largely stable as $N$ increases, confirming that the complexity of the relaxed formulation is primarily governed by the support size rather than by the sample size. 
In contrast, the MILP formulation of Theorem~\ref{thm:main-theorem} exhibits a substantial spread in computation times in the 2D setting (see Figure~\ref{fig:computation-times-thm-small-2D}), indicating a strong sensitivity to the partition-induced MILP structure, with run-times ranging from easily solvable instances to near worst-case behavior.
This computational behavior is representative of all subsequent experiments in this section: for larger support sizes $M$, Proposition~\ref{prop:milp-into-lp} remains tractable on the order of minutes, whereas the formulation in Theorem \ref{thm:main-theorem}  can become impractical. 
While the relaxation in Proposition~\ref{prop:milp-into-lp} is necessarily more conservative than the MILP formulation in Theorem~\ref{thm:main-theorem}, this effect is limited in practice. As we will show in Section \ref{subsection-experiment:comparing-fournier}, for comparable support sizes $M$, the relaxed formulation typically produces bounds within $10-30\%$ of the MILP-based ones, while exhibiting the same qualitative trends.

\begin{figure}[htbp]
    \centering
    \begin{subfigure}[t]{0.39\textwidth}
        \centering
        \includegraphics[width=\linewidth]{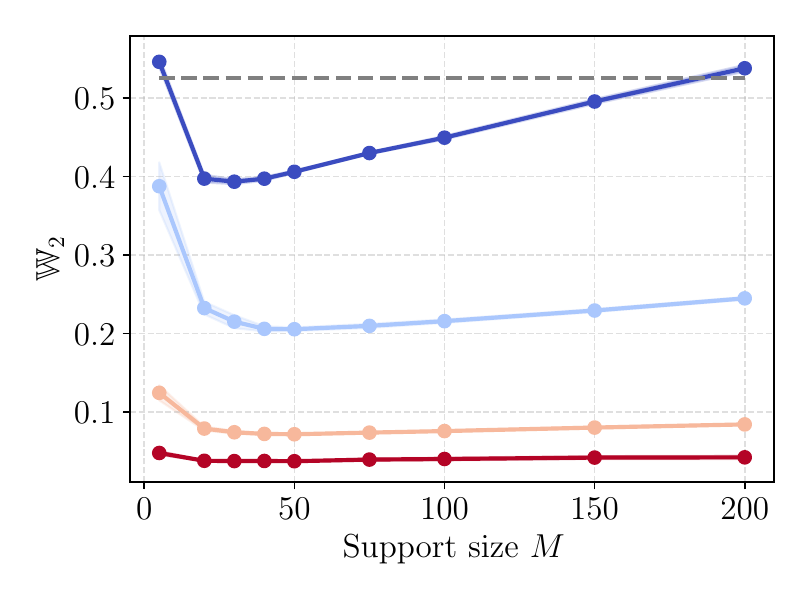}
        \caption{2D - Using Theorem \ref{thm:main-theorem}}
        \label{fig:variance-analysis-a}
    \end{subfigure}
    \begin{subfigure}[t]{0.39\textwidth}
        \centering
        \includegraphics[width=\linewidth]{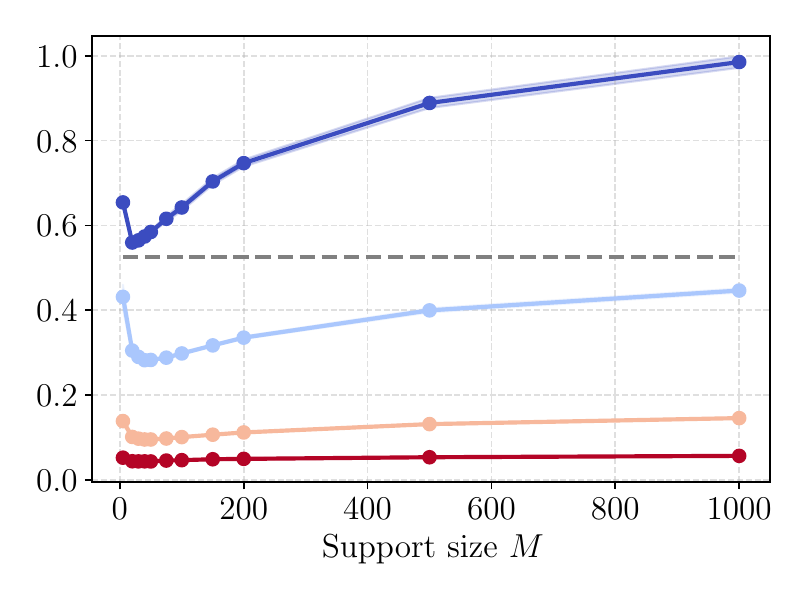}
        \caption{2D - Using Proposition \ref{prop:milp-into-lp}}
        \label{fig:variance-analysis-b}
    \end{subfigure}
    \begin{subfigure}[t]{0.20\textwidth}
        \centering
        \includegraphics[width=\linewidth]{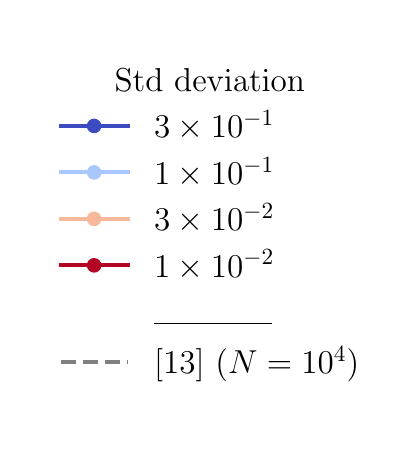}
    \end{subfigure}
    \begin{subfigure}[t]{0.39\textwidth}
        \centering
        \includegraphics[width=\linewidth]{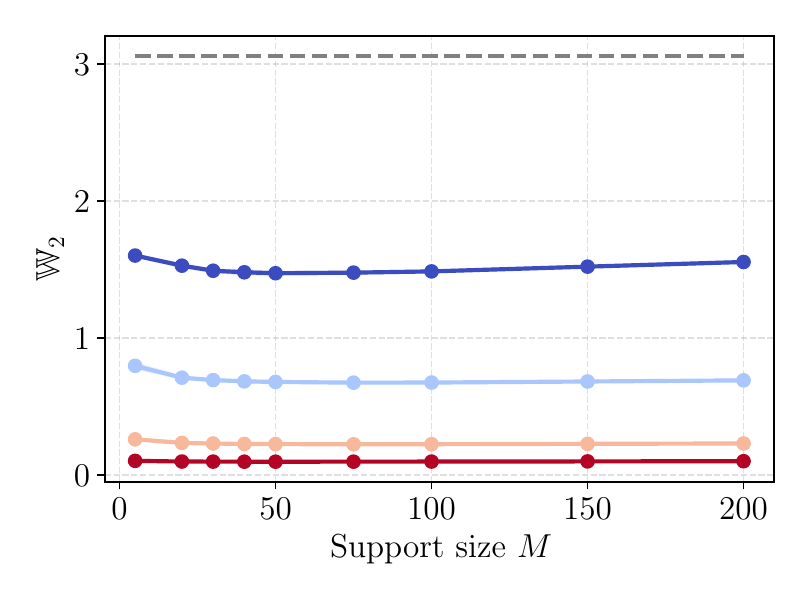}
        \caption{10D - Using Theorem \ref{thm:main-theorem}}
        \label{fig:variance-analysis-c}
    \end{subfigure}
    \begin{subfigure}[t]{0.39\textwidth}
        \centering
        \includegraphics[width=\linewidth]{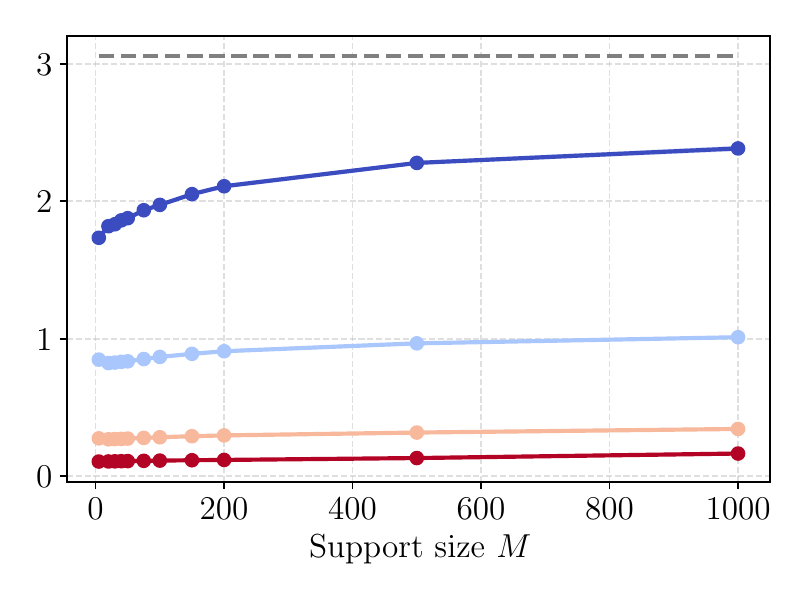}
        \caption{10D - Using Proposition \ref{prop:milp-into-lp}}
    \end{subfigure}
    \begin{subfigure}[t]{0.20\textwidth}
        \phantom{\includegraphics[width=\linewidth]{Figures/Variance/variance_legend.pdf}}
    \end{subfigure}
    \caption{Ours and Fournier bounds (dashed line) for $\rho=2$ for 2D and 10D isotropic Gaussian distributions with decreasing variance for each dimension for $N=10^4$.}
    \label{fig:variance-analysis}
\end{figure}

\subsection{Comparison with the State-of-the-art}\label{subsection-experiment:comparing-fournier}
We now turn to an empirical comparison with state-of-the-art bounds commonly considered and applied in the literature (e.g. \citet{boskos2023high,mohajerin2018data,gracia2024data}), henceforth referred to as \emph{Fournier bound}, as based on the moment bounds introduced in \citet{fournier2015rate} (a detailed derivation of the bounds is reported for completeness in Appendix~\ref{appendix:compare-bounds-existing-literature}).
We should already stress this comparison is not totally fair with respect to our method, as it compares a distribution with support size $N$ (the empirical) against the clusterized empirical distribution produced by our approach, whose support size is $M$ $-$ typically much smaller.\footnote{A more equitable comparison would further compress the empirical distribution to one of size $M$, and then add the discretization error to the Fournier bounds, as in \citet{gracia2024data}. However, this would require solving a linear program with $N\times M$ variables (see \eqref{eq:discrete-wasserstein-lp-optimization}), which would be computationally infeasible for some of the $N$ and $M$ considered here.} Nevertheless, we believe that showing that our approach not only produces distributions with much smaller support but also yields much tighter Wasserstein-distance error bounds would effectively showcase its potential. 
We first consider synthetic distributions with well-controlled structure, which allows us to systematically analyze how properties of the underlying distribution, such as variance and effective dimensionality, affect the resulting bounds. We then extend the comparison to real-world datasets, where the true data-generating distribution is unknown.

\subsubsection{On the Effect of the Variance}\label{subsection-experiment:variance-analysis}
We consider $\Prob$ to be an isotropic Gaussian distribution and analyze the effect of the variance of the underlying distribution~$\Prob$ on the resulting bounds. This provides a controlled setting to isolate how the concentration of probability mass influences the performance of the proposed data-driven bounds.
Notably, the Fournier bounds are agnostic to~$\Prob$ and depend only on the number of samples~$N$ and the ambient dimension, whereas our approach explicitly adapts to the empirical distribution of the data.
As a consequence, our method particularly benefits from low-variance settings, where most of the probability mass is concentrated in a small region of the support. In this regime, the partitioning strategy yields regions with small radii, leading to substantially less conservative bounds.
This behavior is illustrated in Figure~\ref{fig:variance-analysis}, which shows that decreasing the variance results in a marked improvement of our bounds, while the Fournier bounds remain unchanged. Figures~\ref{fig:variance-analysis-a}, \ref{fig:variance-analysis-b}, and~\ref{fig:variance-analysis-c} further highlight the multiscale behavior discussed in the introduction.
In particular, the bound decreases rapidly in the low-$M$ regime, reflecting the ability of the method to exploit coarse quantizations when the distribution is strongly concentrated. The subsequent increase for larger~$M$, which results from the shrinking confidence intervals associated with each partition element, is analyzed in detail in Appendix~\ref{subsection-experiment:finding-optimal-quantization} and is due to the fact that all experiments are performed for a fixed $N=10^4$, whereas, as already discussed in Section \ref{subsection:convergence-guarantees}, to benefit from a larger support $M$, larger $N$ are also needed to mitigate the effect of the use of Clopper-Pearson confidence intervals in each of the resulting $M$ regions.

\subsubsection{On the Effect of the Dimensionality}
We next investigate the effect of the dimension on the relative performance of the proposed bounds and the Fournier bounds. Figure~\ref{fig:comparison-fournier} reports the ratio between our bounds and the Fournier bounds as the dimension increases, for different sample sizes~$N$. 
Despite the unfavorable comparison setup (the $M$ considered in all cases in Figure~\ref{fig:comparison-fournier} is at least two orders of magnitude smaller than $N$, which is also the support of the empirical distribution we are comparing to), our approach consistently outperforms the state-of-the-art across most considered settings. Furthermore, we should stress that to expose the relative dimensional trend, we consider Gaussian distributions with a non-small standard deviation. For smaller standard deviations, as illustrated in the previous section, the ratio would be even smaller already for small $d$.
We note that the ratio between the proposed bounds and the Fournier bounds remains below~$1$ for all but very low-dimensional regimes with extremely large sample sizes, and decreases further as the dimension increases. 
For instance, the improvement reaches nearly a fourfold reduction for the 100D Uniform distribution and up to a ninefold reduction for the Gaussian case.
This behavior can be explained by the different scaling of the two approaches with respect to the dimension. The Fournier bounds scale as $N^{-1/d}$, reflecting the need to discretize the entire space~$\sX$ uniformly, whereas our method refines the discretization only in regions where probability mass is observed. Moreover, unlike the Fournier bounds, the proposed bounds depend explicitly on the underlying distribution~$\Prob$, which explains the different trends observed between the Uniform and Gaussian cases.

\begin{figure}[htbp]
    \centering
    \begin{subfigure}[t]{0.45\textwidth}
        \centering
        \includegraphics[width=\linewidth]{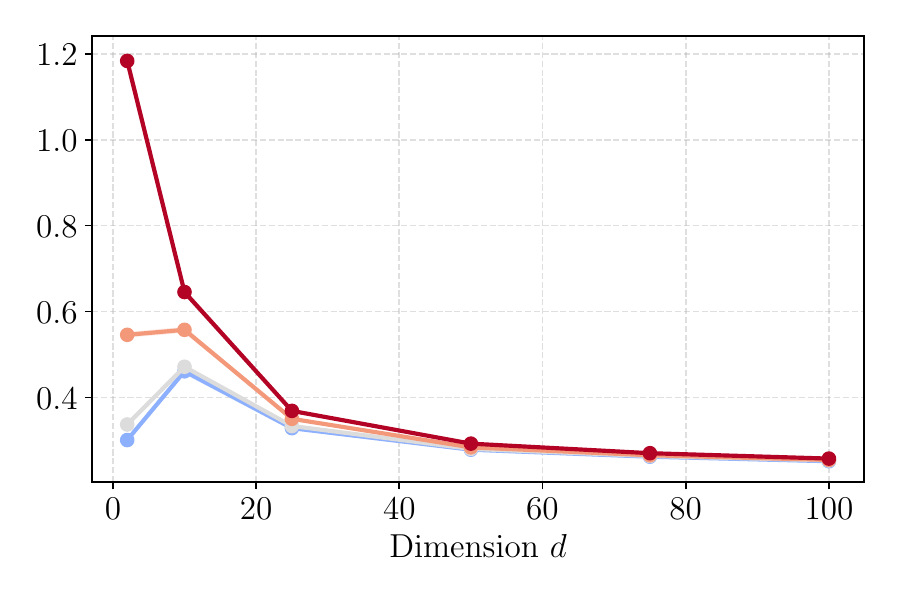}
        \caption{
        \changed{Uniform - Using Theorem \ref{thm:main-theorem}}
        }
        \label{fig:uniform-a}
    \end{subfigure}
    \begin{subfigure}[t]{0.45\textwidth}
        \centering
        \includegraphics[width=\linewidth]{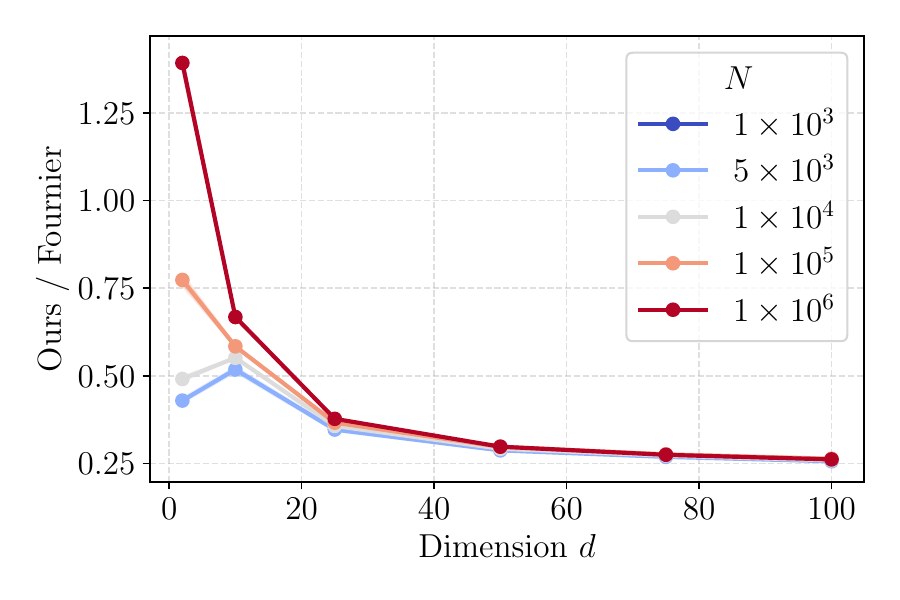}
        \caption{
        \changed{Uniform - Using Proposition~\ref{prop:milp-into-lp}}
        }
        \label{fig:uniform-b}
    \end{subfigure}

    \begin{subfigure}[t]{0.45\textwidth}
        \centering
        \includegraphics[width=\linewidth]{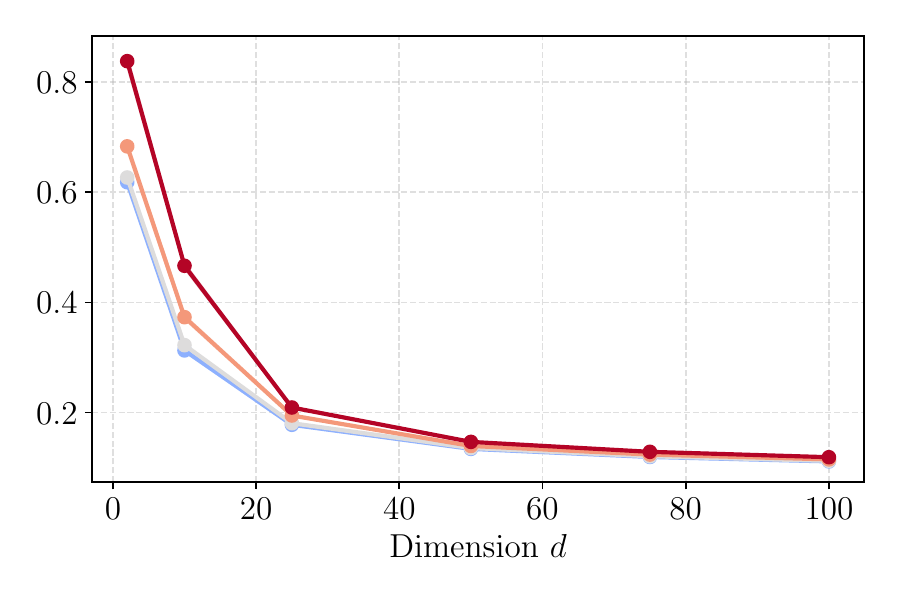}
        \caption{
        \changed{Gaussian - Using Theorem \ref{thm:main-theorem}}
        }
        \label{fig:gaussian-c}
    \end{subfigure}
    \begin{subfigure}[t]{0.45\textwidth}
        \centering
        \includegraphics[width=\linewidth]{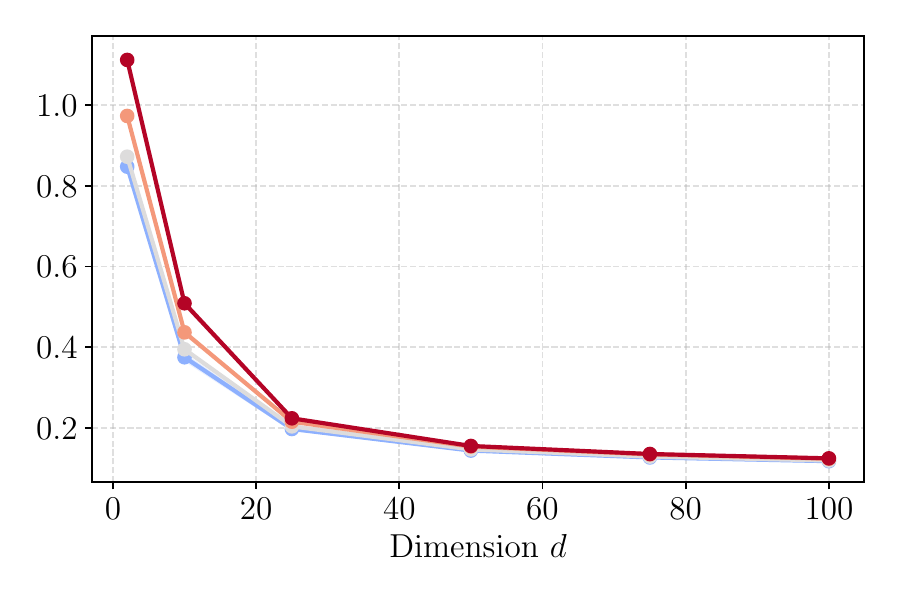}
        \caption{
        \changed{Gaussian - Using Proposition~\ref{prop:milp-into-lp}}
        }
        \label{fig:gaussian-d}
    \end{subfigure}
    \caption{
    The ratio between ours and Fournier bounds for increasing dimension $d$ for a Uniform (for $\rho=1$) and isotropic Gaussian (for $\rho=2$) distribution. 
    In $2$D, the Uniform distribution has an $L_\infty$-ball support of diameter $0.2$, and the Gaussian distribution has standard deviation $0.2$ per dimension. 
    For higher dimensions, the support diameter and variance are scaled to preserve a fixed probability mass ratio w.r.t. the support $\sX$ across dimensions (see Section~\ref{subsection:experimental-details}).
    The support size $M$ is selected from a grid ranging from $5$ to $1000$ so as to minimize the bound for each setting, as reported in Table~\ref{table:uniform_gaussian_over_dims} in the Appendix. 
    }
    \label{fig:comparison-fournier}
\end{figure}

\subsubsection{Application to Real-World Datasets}
We finally evaluate the proposed framework on real-world datasets, where the true data-generating distribution is unknown and the ambient dimension can be large. Table~\ref{tab:datasets} reports the resulting guarantees for two UCI regression benchmark datasets and one safety-critical MedMNIST dataset, and compares them with the corresponding Fournier bounds. 
These results illustrate that classical Wasserstein guarantees that depend explicitly on the ambient dimension are often overly conservative, as real data typically exhibit structure that is not captured by worst-case assumptions.
A particularly illustrative example is the MiniBooNE dataset, where, despite an ambient space of $d=50$, the proposed approach yields a bound of approximately $0.53$ using Theorem~\ref{thm:main-theorem} with a compact support of size $M=100$, compared to a Fournier bound of $11.77$ based on the empirical distribution of $N\approx1.25\times10^5$ samples. 
We also note that OCTMNIST constitutes a particularly challenging benchmark due to its high dimensionality and complex structure, which is reflected in the comparatively large Wasserstein distances reported in Table~\ref{tab:datasets}. Despite this increased difficulty, the proposed bounds remain significantly tighter than the Fournier bounds and with an approximating distribution of support of around three orders of magnitude smaller.

\begin{table}[htbp]
    \centering
    \begin{tabular}{c c c c c c c c} \hline
    Dataset & $d$ & $N_{\mathrm{train}}$ & $N$ & $M$ & Thm.~\ref{thm:main-theorem} & Prop.~\ref{prop:milp-into-lp} & Fournier ($N$) \\ \hline
    Emissions& 11    & 5000  & 31733 & 100   & $0.76 \pm 0.01$  & $1.03 \pm 0.01$ & 2.96 \\
    MiniBooNE & 50    & 5000  & 125064 & 100   & $0.53 \pm 0.01$  & $0.82 \pm 0.01$  & 11.77 \\
    OCTMNIST & 784   & 5000  & 92477 & 100   & $6.92 \pm 0.08$  & $8.66 \pm 0.17$  & 58.87 \\
    \hline
    \end{tabular}
    \caption{
    Ours and Fournier bounds for $\rho=2$ on real-world datasets: Emissions \citep{dos2024co}, MiniBooNE \citep{miniboone_uci}, and OCTMNIST \citep{medmnistv2}.
    }
    \label{tab:datasets}
\end{table}

\section{Conclusion}\label{section:conclusion}

In this work, we introduced a data-driven framework for learning probability distributions from a finite amount of data. Given i.i.d. observations from an underlying distribution $\Prob$, our approach relies on clustering algorithm and optimal transport to construct a discrete approximation $\ProbHat$ whose $\rho$-Wasserstein distance to $\Prob$ can be bounded with high confidence by solving a finite mixed-integer linear optimization problem. Our results show that our approach generally yields substantially less conservative bounds than state-of-the-art methods, while returning distributions of substantially more compact support. 

A central objective of our approach was not only to limit the total number of data used, but also to control the support size of the resulting $\ProbHat$. In fact, this is often critical to guarantee the tractability of downstream computations, such as when the resulting bounds are used in distributionally robust optimization frameworks \citep{mohajerin2018data} or to propagate uncertainty through non-linear functions \citep{figueiredo2025efficient}. To this end, we employed a modified version of Lloyd's algorithm to find our partition of $\sX$. An alternative would have been to instead rely on a dyadic partition of a set containing most of the probability mass of $\Prob$ (for instance, the complement of set $C_M$ introduced in Section \ref{section:partition}). Such a strategy would allow us to leverage known Wasserstein distance bounds for dyadic optimal transport (see, e.g., Proposition 1 in \citet{weed2019sharp}). However, this would also require the construction of a hierarchical partition in cells of decreasing diameter, typically resulting in substantially more than $M$ regions for the error levels considered in our experiments. 
Nevertheless, we see combining our results with existing results for hierarchical optimal transport as an exciting research direction.

This work also opens several additional promising research directions. First, the bound in Theorem \ref{thm:main-theorem} may be significantly tightened by directly solving the min-max optimization problem appearing in \eqref{Eqn:DiscreteTransportInfimumDiscrete}, rather than the relaxation considered in this paper. While this may increase the resulting computational complexity, tailored approaches to solve directly this problem may be beneficial.  Second, while in this work we assumed a bounded support; it remains an interesting future work to extend our results for unbounded (or completely unknown) support. Furthermore, while our approach is tailored for Wasserstein distances, it can also be in principle extended to other optimal transport distances, such as Sinkhorn Divergences \citep{genevay2019sample}. We believe it is an interesting research direction to explore how our data-driven framework can be combined with existing algorithms for these distances. Nevertheless, note that in many safety-critical applications, a bound in Wasserstein distance is often preferred, e.g., to bound worst-case probabilities and expected values of functions \citep{mohajerin2018data}. Specifically, for these applications, which include, for instance, planning \citep{boskos2020data,gracia2025efficient}, financial portfolio optimization \citep{mohajerin2018data}, or out-of-sample prediction tasks \citep{kuhn2019wasserstein,shafieezadeh2015distributionally,chen2018robust}, our results can already have a direct and substantial impact.



\bibliographystyle{plainnat}
\begin{small}
\bibliography{main}
\end{small}

\newpage
\section{Appendix}\label{section:appendix}
\renewcommand{\thefigure}{A\arabic{figure}}
\setcounter{figure}{0}
\renewcommand{\thetable}{A\arabic{table}}
\setcounter{table}{0}

In this appendix, Section \ref{appendix:compare-bounds-existing-literature} reviews the computation of the state-of-the-art bound from the literature, referred to in the paper as the \emph{Fournier bound}. Section \ref{appendix:partition-analysis} then provides a detailed analysis of the empirical performance of Algorithm \ref{alg:approximate-construction}. In Section \ref{subsection-experiment:finding-optimal-quantization}, we present an additional experiment illustrating how to empirically determine the approximately optimal support size $M$ for $\ProbHat$. Section \ref{subsection:experimental-details} provides experimental details, while Section \ref{subsec:appendix-dim-scaling} contains a tabulated summary of results across multiple settings. Lastly, Section \ref{section:proofs} provides the proofs of all mathematical results in the paper.

\subsection{Computing the State-of-the-art bound from the literature}\label{appendix:compare-bounds-existing-literature}
We recall the computation of the state-of-the-art finite-sample bound commonly used in applications (e.g., \citet{boskos2023high,gracia2024data}), which we refer to as the \emph{Fourier bound} in the main text. Throughout this section, let $\ProbHat$ denote the empirical distribution constructed from $N$ i.i.d. samples drawn from a distribution $\Prob \in \sP(\sX)$, where $\sX\subset \realNum^d$ is compact. From Proposition A.2 in \citet{boissard2014mean}, for any $\tau\geq0$, it holds that
\begin{align}\label{eq:concentration-ineq-boissard}
    \Prob^N\big( \wasserstein_\rho(\Prob, \ProbHat) \geq \expect[\sT_\rho(\Prob, \ProbHat)]^\frac{1}{\rho} + \tau \big) \leq e^{-N\tau^{2\rho}/(2\norm{\sX}^{2\rho})},
\end{align}
where we also use that $\expect[\wasserstein_\rho(\Prob, \ProbHat)] \leq \expect[\sT_\rho(\Prob, \ProbHat)]^\frac{1}{\rho}$ by Jensen's inequality. Practitioners then employ the moment bound for $\expect[\sT_\rho(\Prob, \ProbHat)]^\frac{1}{\rho}$ introduced in \citet{fournier2015rate} for which numerical values are tabulated in Section 2.5 of \citet{fournier2023convergence}. Specifically, Tables 1 and 2 in \citet{fournier2023convergence} provide bounds for both $\expect[\sT_1(\Prob, \ProbHat)]$ and $\expect[\sT_2(\Prob, \ProbHat)]^\frac{1}{2}$ when the support $\sX = \sB_\infty(0, 1/2)$ or $\sX = \sB_2(0, 1/2)$ (as $\sB_2(0, 1/2) \subset \sB_\infty(0, 1/2)$)\footnote{In fact, Tables 1 and 2 present bounds with respect to the $L_\infty$-norm. The corresponding $L_2$-norm bounds used in our work are obtained by multiplying by a factor $\sqrt{d}$. Moreover, numerical values for $\rho\in \{1, 2\}$ are sufficient for our purposes, as these are the Wasserstein orders typically used in applications.}. For the case $\sX = \sB_2(0, 1/2)$, Tables 3 and 4 of \citet{fournier2023convergence} further refine these bounds for several representative dimensions $d$. Let $\sE_\rho\geq 0$ denote the tabulated upper bound satisfying $\expect[\sT_\rho(\Prob, \ProbHat)]^\frac{1}{\rho}\leq \sE_\rho$. Given a confidence level $1-\beta>0$, if we select $\tau = \norm{\sX} \left( \frac{2 \log\beta^{-1}}{N} \right)^\frac{1}{2\rho}$, then from \eqref{eq:concentration-ineq-boissard}, we have that
\begin{align}
    \Prob^N\Big( \wasserstein_\rho(\Prob, \ProbHat) \leq \sE_\rho + \norm{\sX}\left( \frac{2 \log\beta^{-1}}{N} \right)^\frac{1}{2\rho} \Big) \geq 1-\beta.
\end{align}

\begin{figure}[htbp]
    \centering
    \begin{subfigure}[t]{0.32\textwidth}
        \centering
        \includegraphics[width=\linewidth]{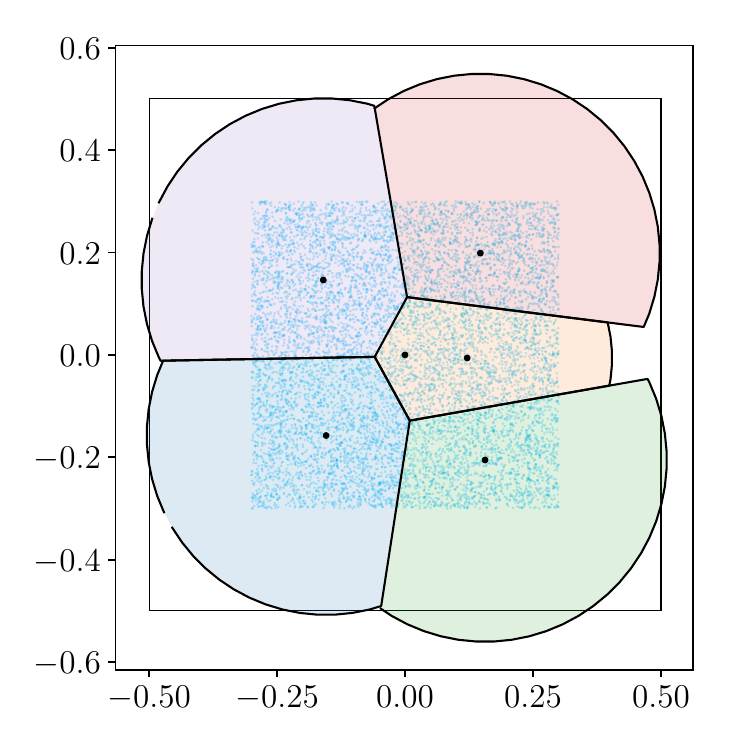}
        \caption{$M=5$}
    \end{subfigure}
    \begin{subfigure}[t]{0.32\textwidth}
        \centering
        \includegraphics[width=\linewidth]{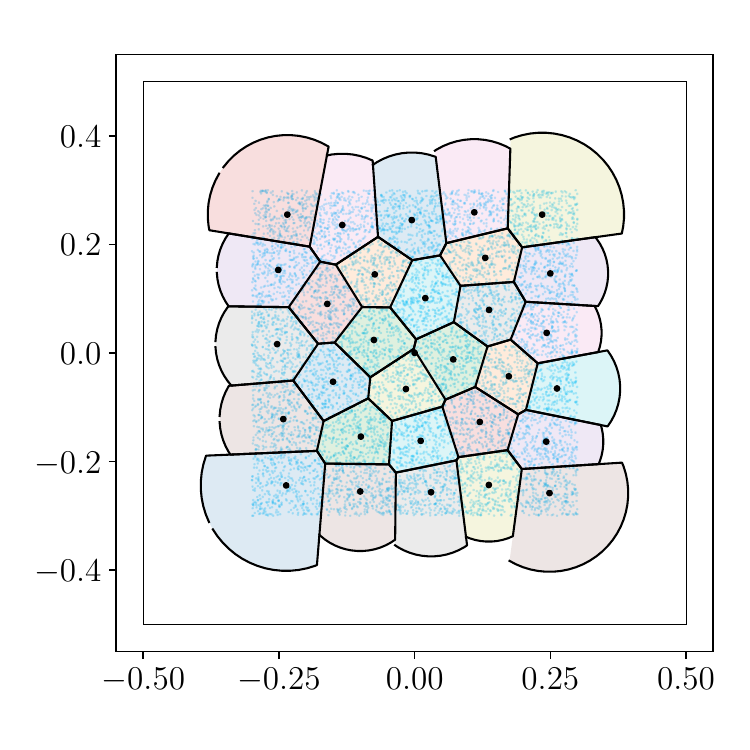}
        \caption{$M=30$}
    \end{subfigure}
    \begin{subfigure}[t]{0.32\textwidth}
        \centering
        \includegraphics[width=\linewidth]{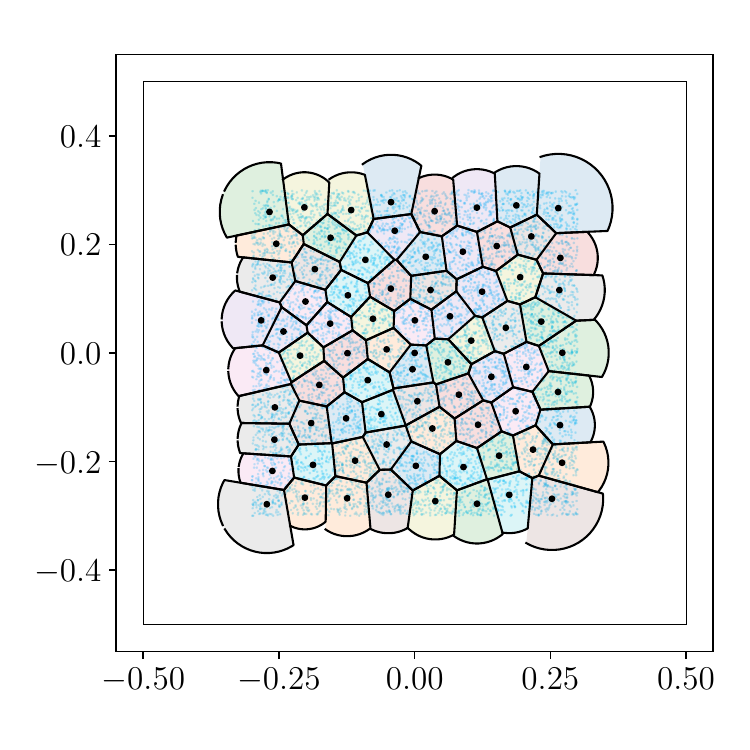}
        \caption{$M=75$}
    \end{subfigure}
    \caption{
    Illustration of the resulting representation points $\{\vc_i\}_{i=1}^M$ (black dots) and partition $\{C_i\}_{i=1}^M$ for different values of $M$ using $N_{\mathrm{train}}=5\times 10^3$ samples from an unknown 2D Uniform distribution with true support $[0.3, 0.3]^2$ and assumed support $\sX=[0.5, 0.5]^2$ (black box). The independent dataset $\sD_N$ with $N_{\mathrm{train}}=10^4$ (blue dots) is shown for validation of the resulting partition.
    }
    \label{fig:example_quantizations_uniform}
\end{figure}

\subsection{Analyzing the empirical performance of Algorithm~\ref{alg:approximate-construction}}\label{appendix:partition-analysis}
The discrete approximation $\ProbHat$ of $\Prob$ is constructed according to Algorithm~\ref{alg:approximate-construction} using two independent sample sets $\sD_{N_{\mathrm{train}}}$ and $\sD_{N}$, with the objective of minimizing $\wasserstein_\rho(\Prob,\ProbHat)$. 
In this section, we empirically evaluate the effectiveness of this construction and provide quantitative evidence supporting the heuristic choices introduced in Section~\ref{section:partition}.
The analysis is carried out on a 2D Gaussian mixture and a 2D Uniform distribution, with the corresponding partitions constructed by Algorithm~\ref{alg:approximate-construction} for varying $M$ shown in Figures~\ref{fig:example_quantizations} and~\ref{fig:example_quantizations_uniform}, respectively.

\subsubsection{Approximating the true support of the unknown distribution $\Prob$}
As discussed in Section~\ref{section:partition}, minimizing $\wasserstein_\rho(\Prob,\ProbHat)$ corresponds to choosing regions $\{\sC_1,\ldots,\sC_{M-1}\}$ that form a (Voronoi) partition of the true support of $\Prob$.
Since the true support is unknown, Algorithm~\ref{alg:approximate-construction} approximates it as the union of $L_m$-balls centered at the representative points $\{\vc_1,\ldots,\vc_{M-1}\}$, with associated radii $\{r_1,\ldots,r_{M-1}\}$, denoted by $\sX_{\mathrm{approx}}$.
The radii are constructed according to the heuristic described in lines~3-5 of Algorithm~\ref{alg:approximate-construction}, where the hyperparameter value $0.05$ used in line~4 is selected via a grid search over a range of settings and fixed a priori for all experiments reported in this work.
To assess whether this approximation adequately covers the true support, we measure the fraction of $N$ samples drawn independently of those used to construct $\sX_{\mathrm{approx}}$ that are contained in the uncovered region $\sX\setminus\sX_{\mathrm{approx}}$.
Figures~\ref{fig:partition-analysis-counts} show that, for both the Gaussian mixture and Uniform settings, no samples are observed in the uncovered region across all considered partition sizes $M$ and number of samples $N$.
This observation indicates that the radii $\{r_1,\ldots,r_{M-1}\}$ are chosen sufficiently large to cover the true support of $\Prob$ with high probability. 

While coverage of the true support could trivially be guaranteed by setting each $r_i$ equal to the diameter of the assumed support $\sX$, yielding $\sX_{\mathrm{approx}} = \sX$, this choice would be overly conservative. Our objective instead is to learn an approximation of the true support of $\Prob$, which may be strictly contained in $\sX$, as illustrated by the 2D Uniform setting in Figure~\ref{fig:example_quantizations_uniform}. A conservative approximation of the support leads to looser upper bounds on $\wasserstein_\rho(\Prob,\ProbHat)$ in Theorem~\ref{thm:main-theorem}, since these bounds depend on terms of the form $\max_{\vx \in \sC_i}\|\vx - \vc_j\|$.
From Figure~\ref{fig:example_quantizations_uniform}, we observe that in the uniform setting $\sX_{\mathrm{approx}}$ tightly encloses the true support of $\Prob$, with the approximation becoming increasingly tight as $M$ grows. For the Gaussian mixture setting shown in Figure~\ref{fig:example_quantizations}, where the true support coincides with the assumed support $\sX$, $\sX_{\mathrm{approx}}$ instead concentrates on regions of high probability mass. Since any uncovered region is explicitly accounted for through the additional region $\sC_M$, this behavior does not affect correctness and is in fact desirable, as it avoids unnecessarily enlarging the regions $\sC_i$ in areas with negligible probability mass. Figures~\ref{fig:partition-analysis-radii} quantify these observations by showing that the average radius of the regions $\{\sC_1,\ldots,\sC_{M-1}\}$ decreases as the support size $M$ increases.

\subsubsection{Radii of the partition regions}
Recall that the regions $\{\sC_1,\ldots,\sC_{M-1}\}$ are defined as the Voronoi partition induced by the representative points $\{\vc_1,\ldots,\vc_{M-1}\}$ on $\sX_{\mathrm{approx}}$.
Equivalently, these regions can be interpreted as Voronoi cells on $\sX$, each truncated by an $L_m$-ball of radius $r_i$ centered at $\vc_i$.
By construction, if all neighboring representative points are closer to $\vc_i$ than $r_i$, the truncation is inactive, and the radius of the smallest $\ell_p$-ball enclosing $\sC_i$, given by $\max_{\vx\in\sC_i}\|\vx-\vc_i\|$, can be strictly smaller than $r_i$.
Interestingly, Figures~\ref{fig:partition-analysis-radii} show that, in practice, the values $r_i$ closely approximate $\max_{\vx\in\sC_i}\|\vx-\vc_i\|$ across regions.
We exploit this observation to efficiently bound the cost terms $\max_{\vx\in\sC_i}\|\vx-\vc_j\|$ appearing in Theorem~\ref{thm:main-theorem}, whose exact computation becomes prohibitive for large $M$ and high dimensions. In particular, by a single application of the triangle inequality,
\[
    \max_{\vx\in\sC_i}\|\vx-\vc_j\| \leq \max_{\vx\in\sC_i}\|\vx-\vc_i\| + \|\vc_i-\vc_j \| \leq r_i + \|\vc_i-\vc_j \|.
\]
As suggested by Figures~\ref{fig:partition-analysis-radii}, the second inequality is typically tight. 
Moreover, the pairwise distances $\|\vc_i-\vc_j\|$ are obtained directly as a byproduct of the $(M-1)$-means clustering step in Algorithm~\ref{alg:approximate-construction}, resulting in a computationally efficient and empirically accurate bound. 
Figures~\ref{fig:partition-analysis-l2-distance} further show that, as $M$ increases, the average pairwise distances $\|\vc_i-\vc_j\|$ remain relatively stable, so that our bounds on the terms $\max_{\vx\in\sC_i}\|\vx-\vc_j\|$ diminish due to decreasing radii $r_i$.

\subsubsection{Concentration of regional probability mass intervals}
Having constructed the regions $\sC_1,\hdots,\sC_M$, Algorithm~\ref{alg:approximate-construction} computes the probability vector $\vpi$ defining $\ProbHat$ in lines~6-7 by counting the number of samples in the dataset $\sD_N$ that fall in each region $\sC_i$. 
Figure~\ref{fig:partition-analysis-counts} accordingly reports the average probability per region for the settings considered.
Based on this probability vector, Theorem~\ref{thm:main-theorem} constructs confidence intervals $[\evlowerb{i}, \evupperb{i}]$ for the probability mass of each region according to \eqref{eq:clopper-pearson-lower-bound} and \eqref{eq:clopper-pearson-upper-bound}. 
Figures~\ref{fig:partition-analysis-probs} show that, for a fixed partition, increasing the number of samples leads to a systematic reduction in the width of these probability mass intervals. This behavior confirms that the uncertainty in the regional probability estimates decreases with additional data, directly translating into tighter Wasserstein bounds through Theorem~\ref{thm:main-theorem}.

\begin{figure}
    \centering
    \begin{subfigure}[t]{\textwidth}
        \centering
        \begin{minipage}{0.48\textwidth}
            \centering
            \includegraphics[width=0.8\linewidth]{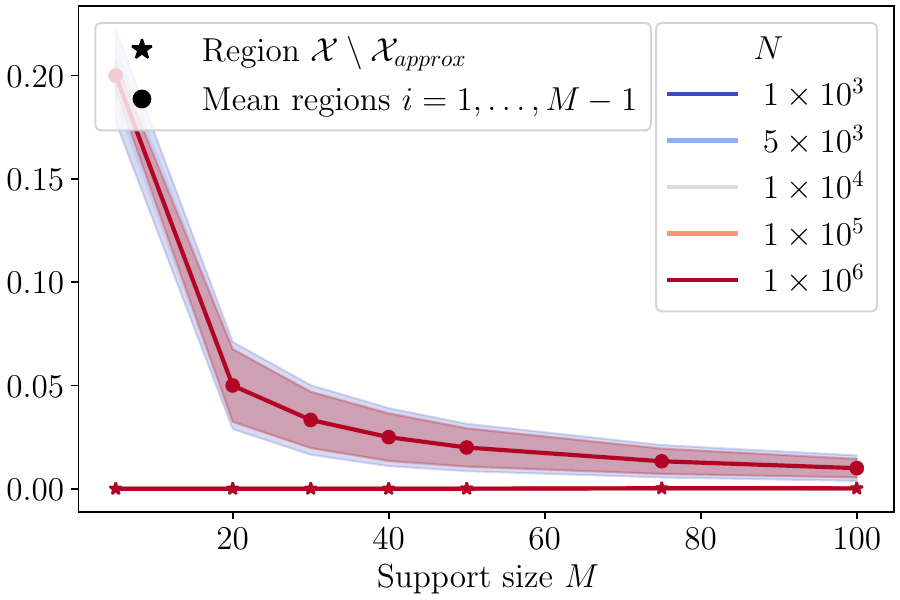} 
        \end{minipage}
        \hfill
        \begin{minipage}{0.48\textwidth}
            \centering
            \includegraphics[width=0.8\linewidth]{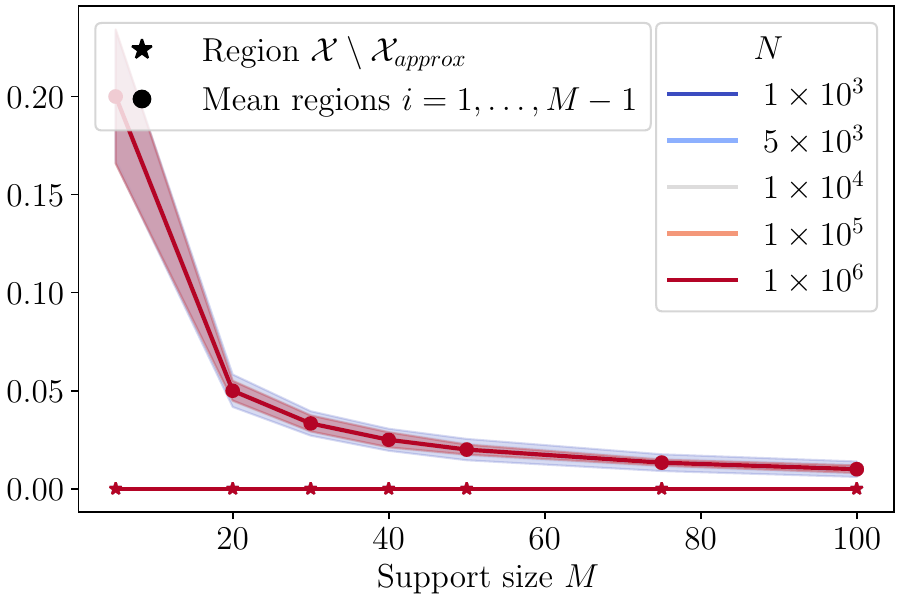}
        \end{minipage}
        \caption{Fraction of $N$ samples falling in the region $\sX\setminus\sX_{\mathrm{approx}}$ and average over the remaining regions.}
        \label{fig:partition-analysis-counts}
    \end{subfigure}
    \vspace{1em}
    \begin{subfigure}[t]{\textwidth}
        \centering
        \begin{minipage}{0.48\textwidth}
            \centering
            \includegraphics[width=0.8\linewidth]{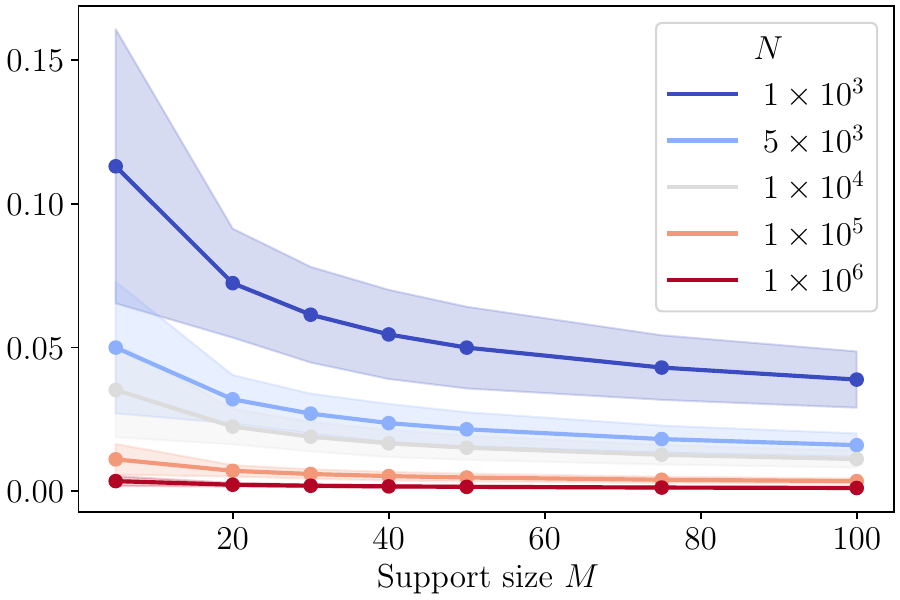}
        \end{minipage}
        \hfill
        \begin{minipage}{0.48\textwidth}
            \centering
            \includegraphics[width=0.8\linewidth]{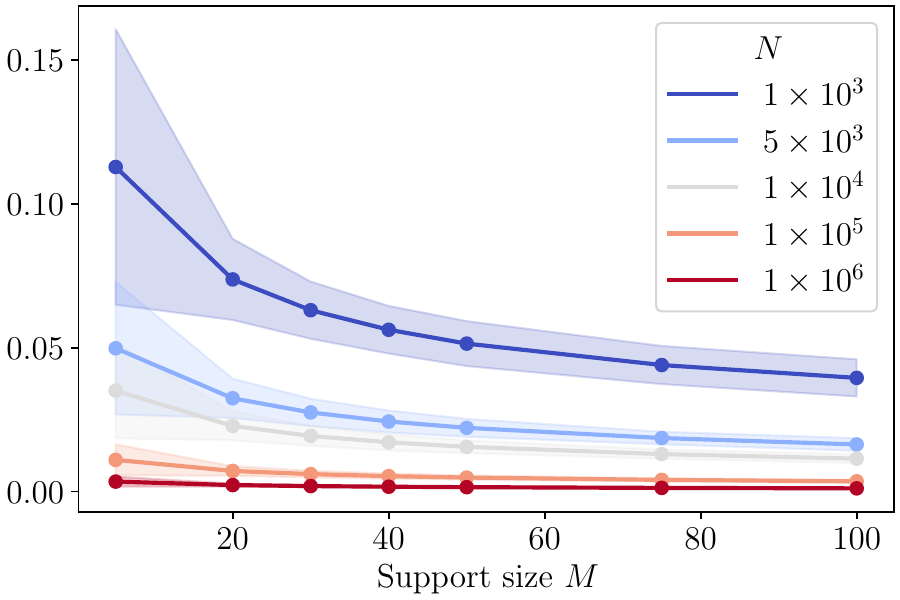}    
        \end{minipage}
        \caption{Width of the probability mass interval $[\evlowerb{i}, \evupperb{i}]$. }
        \label{fig:partition-analysis-probs}
    \end{subfigure}
    \vspace{0.5em}
    \begin{subfigure}[t]{\textwidth}
        \centering
        \begin{minipage}{0.48\textwidth}
            \centering
            \includegraphics[width=0.8\linewidth]{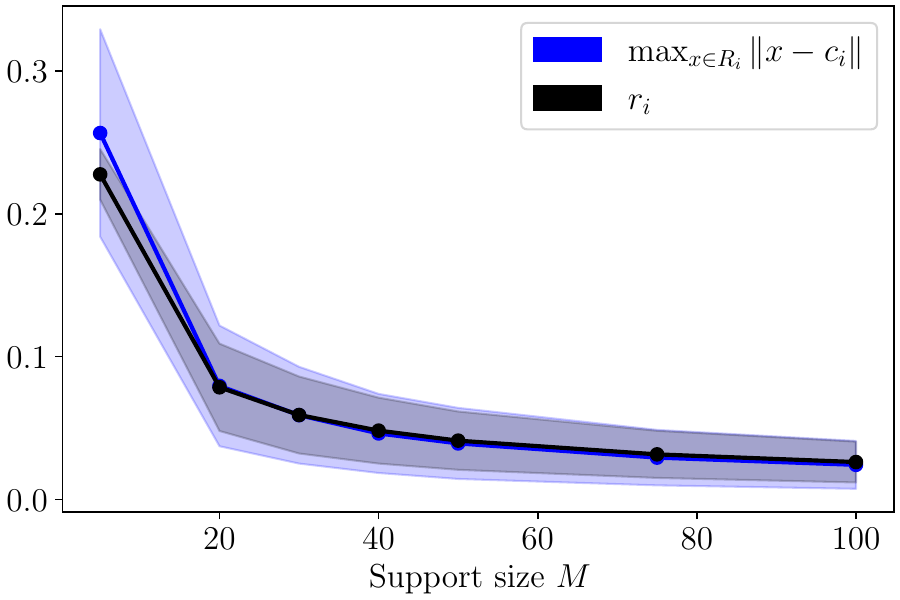}    
        \end{minipage}
        \hfill
        \begin{minipage}{0.48\textwidth}
            \centering
            \includegraphics[width=0.8\linewidth]{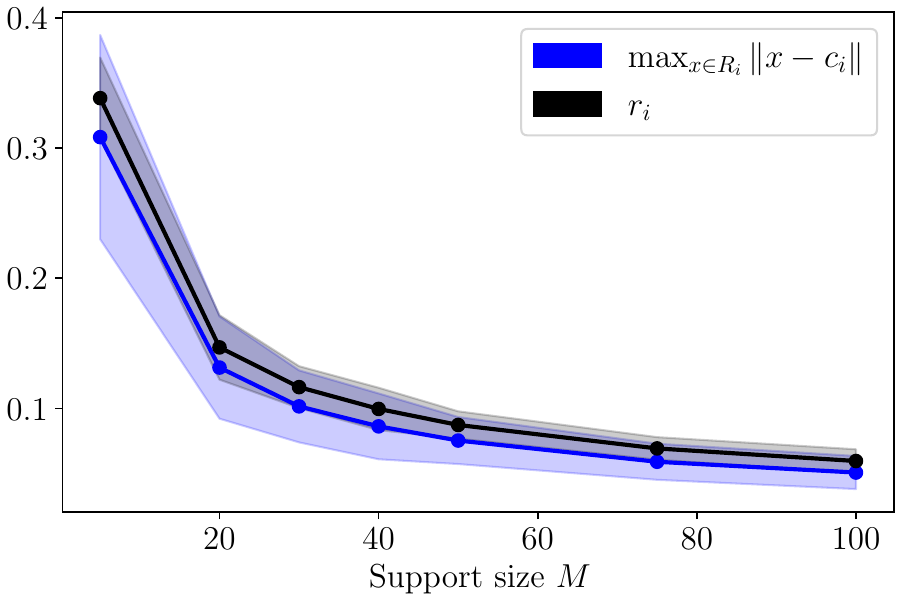}    
        \end{minipage}
        \caption{Heuristic radii $r_i$ compared with exact region radii $\max_{\vx\in\sC_i}\|\vx-\vc_i\|$}
        \label{fig:partition-analysis-radii}
    \end{subfigure}
    \vspace{0.5em}
    \begin{subfigure}[t]{\textwidth}
        \centering
        \begin{minipage}{0.48\textwidth}
            \centering
            \includegraphics[width=0.8\linewidth]{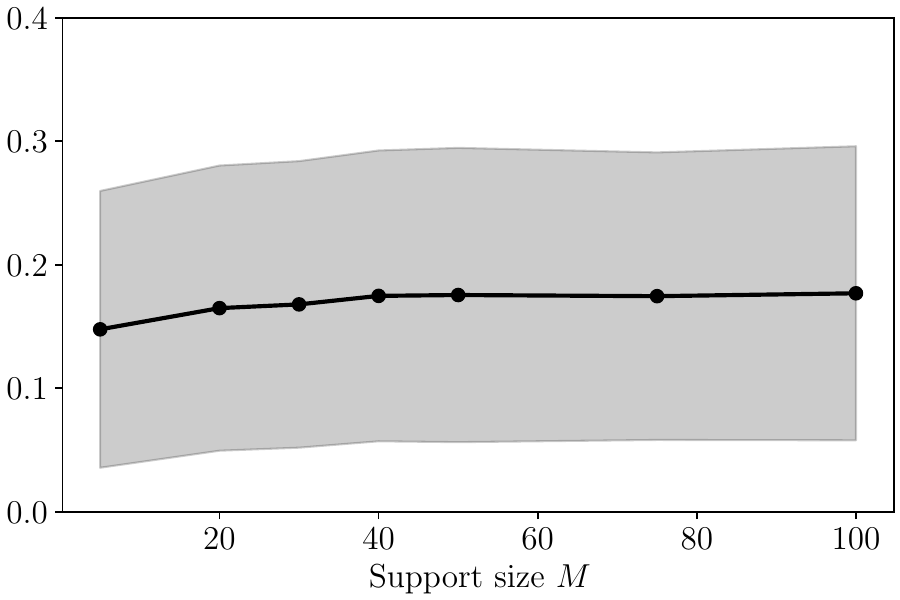}
        \end{minipage}
        \hfill
        \begin{minipage}{0.48\textwidth}
            \centering
            \includegraphics[width=0.8\linewidth]{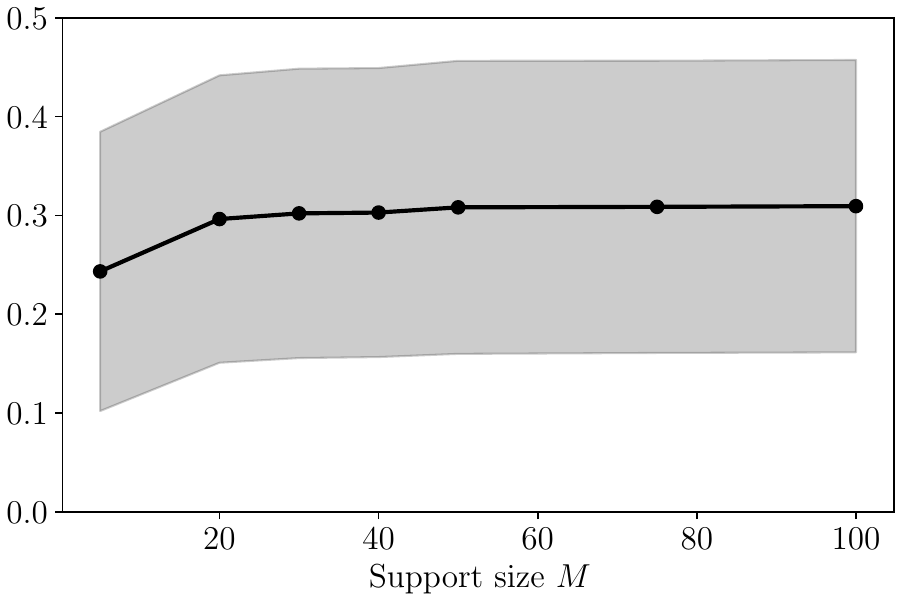}
        \end{minipage}    
        \caption{Pairwise distance $\|c_i-c_j\|$ ($i\neq j$)}
        \label{fig:partition-analysis-l2-distance}
    \end{subfigure}
    \caption{Empirical analysis of Algorithm~\ref{alg:approximate-construction} for the 2D Gaussian mixture (left column) and 2D Uniform (right column) settings. Solid lines denote mean values, and shaded areas indicate one standard deviation. }
    \label{fig:partition-analysis}
\end{figure}

\subsection{On the Choice of the Support Size}\label{subsection-experiment:finding-optimal-quantization}
We provide an additional analysis on the effect of the quantization support size $M$ on the bounds in Theorem \ref{thm:main-theorem} and Proposition \ref{prop:milp-into-lp}. In particular, we show that we can empirically select the optimal quantization $\ProbHat$ by observing the inflection point of the bound w.r.t. $M$, as seen in Figure \ref{fig:inflection-point-m}. For a fixed number of samples $N$, increasing the support size has a dual effect: while a higher $M$ contributes for the reduction of the quantization error, it also reduces the confidence $\frac{\beta}{M}$ associated to each region $C_i$ (see \eqref{eq:clopper-pearson-lower-bound} and \eqref{eq:clopper-pearson-upper-bound}), increasing the length of the intervals $[\evlowerb{i}, \evupperb{i}]$. While the first effect dominates for low $M$ (due to the multiscale behavior discussed in the introduction of this work), the latter becomes more relevant as $M$ grows, thus creating the inflection that allows us to empirically select the optimal support size $M$.
\begin{figure}[ht]
    \centering
    \begin{subfigure}[t]{0.45\textwidth}
        \centering
        \includegraphics[width=\linewidth]{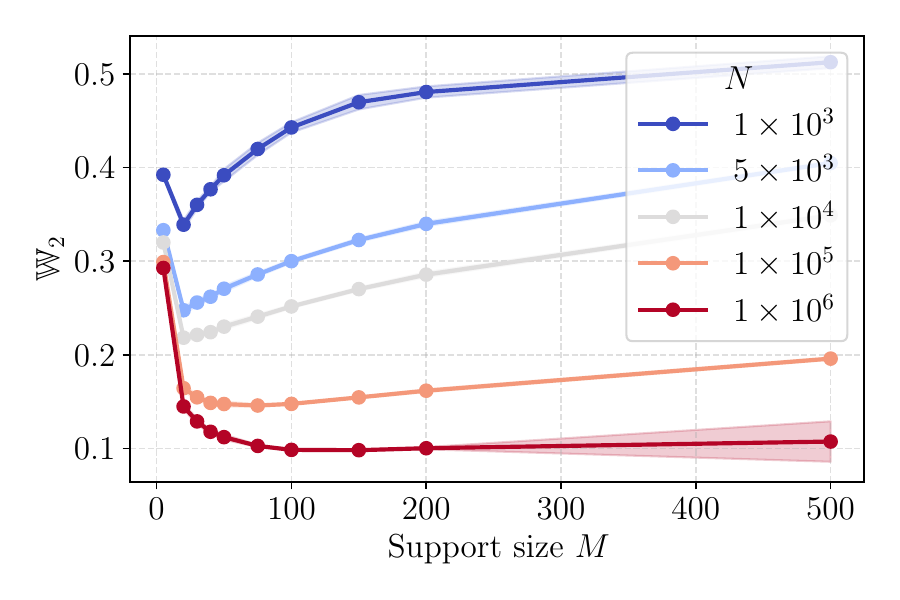}
        \caption{Using Theorem \ref{thm:main-theorem}}
    \end{subfigure}
    \begin{subfigure}[t]{0.45\textwidth}
        \centering
        \includegraphics[width=\linewidth]{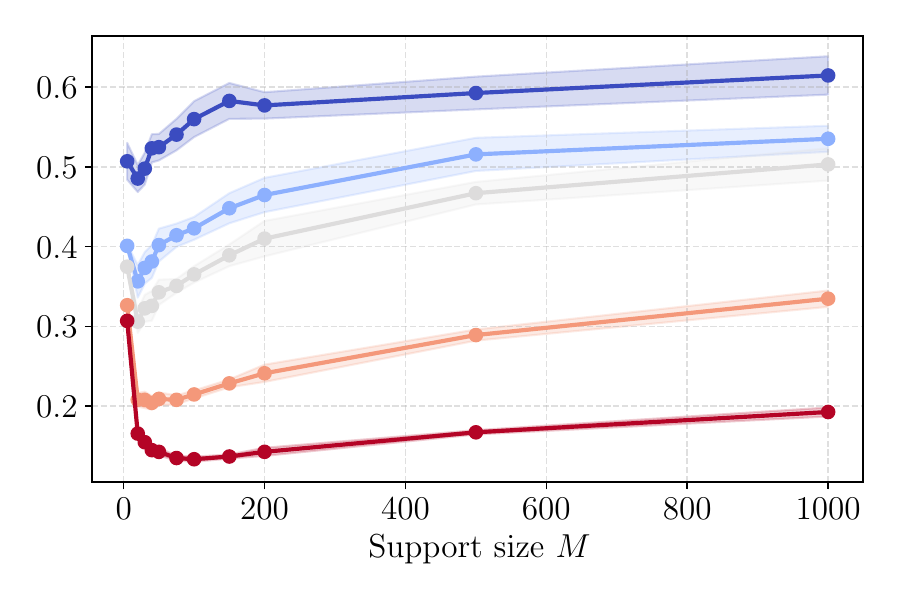}
        \caption{Using Proposition \ref{prop:milp-into-lp}}
    \end{subfigure}
    \caption{Our bound (for $\rho=2$) as the support size $M$ increases for multiple number of samples $N$ in a 3D Gaussian Mixture setting.}
    \label{fig:inflection-point-m}
\end{figure}

\subsection{Experimental Details}\label{subsection:experimental-details}
In the following, we provide details on the distributions considered in each experiment in Section \ref{section:experiments}. In Figure \ref{fig:computation-time}, we consider Gaussian distributions described by $\Prob = \sN(\bm{0}, 3 \times 10^{-2} I_{2\times 2})$ and $\Prob = \sN(\bm{0}, 2.2 \times 10^{-2} I_{100\times 100})$. In Figure \ref{fig:variance-analysis}, we also consider Gaussian distributions given by $\Prob = \sN(\bm{0}, \sigma^2I_{d\times d})$, where the variance $\sigma^2$ is indicated in the legend. In Figure \ref{fig:comparison-fournier}, we consider Uniform distributions with support conserving the same volume ratio w.r.t. to $\sX$. For dimension $d=2$, we set the support of the Uniform distribution to be a $L_\infty$-ball of diameter $0.2$, leading to a volume ratio of $\frac{0.2^2}{1^2}=0.04$, which is then used as a reference for the remaining dimensions. For instance, for dimension $d=10$, the diameter $\phi$ is chosen so that $\frac{\phi^{10}}{1^{10}}=0.04$, thus $\phi \approx 0.72$. An analogous reasoning is applied for the Gaussian distributions. For $d=2$, we set the covariance to $3\times 10^{-2} I_{2\times 2}$, for which $\sX$ contains a probability mass of approximately $0.992$. This is the quantity that we choose to conserve for the remaining dimensions. For instance, for $d=10$, a variance of $\sigma^2 \approx 0.0221$ is required for the probability criterion to be met\footnote{For $Z = (Z_1,\dots,Z_d)\sim \sN(\bm{0},\sigma^2I_{d\times d})$, $\Prob_Z(\sX)=\prod_{i=1}^d \Prob_{Z_i}([-0.5,0.5])$ as independence can be concluded from the nonexistence of correlation in this case. Because $\Prob_{Z_i}([-0.5,0.5])=2\Phi\left( \frac{1}{2\sigma} \right)-1$, it holds that $\Prob_Z(\sX) = \Big( 2\Phi\left( \frac{1}{2\sigma} \right)-1 \Big)^d$, where $\Phi$ is the CDF of the standard Gaussian, which allows us to compute the variance $\sigma^2$ required for any given desired value for $\Prob_Z(\sX)$.}.

\subsection{Tabulated Results for Dimensional Scaling Experiments}\label{subsec:appendix-dim-scaling}
\setlength{\tabcolsep}{3pt} 
\begin{table*}[h]
\centering
\footnotesize 
\begin{subtable}{0.48\textwidth}
    \centering   \centering
\vspace{0.3cm}
\begin{tabular}{c c c c c c}
\hline
$d$ & $N$ & $M_{\text{opt}}$ & Thm.~\ref{thm:main-theorem} & Prop.~\ref{prop:milp-into-lp} & Fournier \\
\hline
$2$ & $10^{3}$ & $20$ & $0.122$ & $0.174$ & $0.505$ \\
 & $10^{4}$ & $50$ & $0.061$ & $0.090$ & $0.184$ \\
 & $10^{5}$ & $100$ & $0.034$ & $0.049$ & $0.066$ \\
 & $10^{6}$ & $100$ & $0.025$ & $0.031$ & $0.023$ \\
\hline
$10$ & $10^{3}$ & $20$ & $1.320$ & $1.438$ & $3.394$ \\
 & $10^{4}$ & $75$ & $1.147$ & $1.338$ & $2.445$ \\
 & $10^{5}$ & $200$ & $0.895$ & $1.056$ & $1.863$ \\
 & $10^{6}$ & $200$ & $0.796$ & $0.757$ & $1.454$ \\
\hline
$50$ & $10^{3}$ & $5$ & $3.375$ & $3.539$ & $13.185$ \\
 & $10^{4}$ & $30$ & $3.318$ & $3.466$ & $11.841$ \\
 & $10^{5}$ & $200$ & $3.077$ & $3.286$ & $11.070$ \\
 & $10^{6}$ & $1000$ & $2.752$ & $2.926$ & $10.497$ \\
\hline
$100$ & $10^{3}$ & $5$ & $4.801$ & $4.991$ & $20.141$ \\
 & $10^{4}$ & $40$ & $4.741$ & $4.840$ & $18.583$ \\
 & $10^{5}$ & $100$ & $4.519$ & $4.699$ & $17.813$ \\
 & $10^{6}$ & $500$ & $4.319$ & $4.420$ & $17.298$ \\
\hline
\end{tabular}
\label{table:comparison-fournier-uniform}
    \caption{Uniform ($\rho=1$)}
\end{subtable}
\hfill
\begin{subtable}{0.48\textwidth}
    \centering
    \centering
\vspace{0.3cm}
\begin{tabular}{c c c c c c}
\hline
$d$ & $N$ & $M_{\text{opt}}$ & Thm.~\ref{thm:main-theorem} & Prop.~\ref{prop:milp-into-lp} & Fournier \\
\hline
$2$ & $10^{3}$ & $20$ & $0.518$ & $0.693$ & $0.934$ \\
 & $10^{4}$ & $40$ & $0.334$ & $0.464$ & $0.525$ \\
 & $10^{5}$ & $100$ & $0.201$ & $0.291$ & $0.295$ \\
 & $10^{6}$ & $100$ & $0.131$ & $0.182$ & $0.166$ \\
\hline
$10$ & $10^{3}$ & $20$ & $1.133$ & $1.292$ & $4.221$ \\
 & $10^{4}$ & $75$ & $0.983$ & $1.193$ & $3.054$ \\
 & $10^{5}$ & $500$ & $0.751$ & $0.967$ & $2.258$ \\
 & $10^{6}$ & $200$ & $0.683$ & $0.728$ & $1.699$ \\
\hline
$50$ & $10^{3}$ & $5$ & $1.842$ & $2.000$ & $14.892$ \\
 & $10^{4}$ & $40$ & $1.777$ & $1.946$ & $13.090$ \\
 & $10^{5}$ & $200$ & $1.618$ & $1.816$ & $11.864$ \\
 & $10^{6}$ & $500$ & $1.486$ & $1.630$ & $10.972$ \\
\hline
$100$ & $10^{3}$ & $5$ & $2.422$ & $2.579$ & $22.556$ \\
 & $10^{4}$ & $50$ & $2.294$ & $2.408$ & $20.351$ \\
 & $10^{5}$ & $200$ & $2.162$ & $2.334$ & $18.936$ \\
 & $10^{6}$ & $500$ & $2.052$ & $2.210$ & $17.970$ \\
\hline
\end{tabular}
\label{table:comparison-fournier-gaussian}
    \caption{Gaussian ($\rho=2$)}
\end{subtable}
\caption{Summary of results for various settings. The quantity $M_{\text{opt}}$ is the optimal support size for $\ProbHat$ obtained empirically for each given $N$ (as explained in Appendix \ref{subsection-experiment:finding-optimal-quantization}).
}
\label{table:uniform_gaussian_over_dims}
\end{table*}
\clearpage
\subsection{Proofs}\label{section:proofs}
We provide the proofs of the main results together with supporting technical results.

\subsubsection{Proof of Theorem \ref{thm:main-theorem}}
The main ingredient of the proof is the following Proposition.
\begin{proposition}
\label{prop:wasserstein-bound-no-confidence}
    Assume that for any $C_i\in \{C_1,...,C_M\}$ it holds that $0\leq \evlowerb{i} \leq \Prob(C_i) \leq \evupperb{i} \leq 1.$ Then,
    $$ \wasserstein_{\rho}(\Prob,\ProbHat) \leq \epsilon(\samplesVector), $$
where $\epsilon(\samplesVector)$ is as defined in Theorem \ref{thm:main-theorem}.    
\end{proposition}
\begin{proof}
The proof starts by noticing that because of the assumption $\evlowerb{i} \leq \Prob(C_i) \leq \evupperb{i}$, it holds that
\begin{align}
    \sT_{\rho}(\Prob,\ProbHat) &\leq \sup_{\Prob \in \sP(\sX)} \sT_{\rho}(\Prob,\ProbHat) \nonumber \\
    &\quad \text{ s.t. } \Prob(C_i) \in [\evlowerb{i}, \evupperb{i}] \text{ for all } i \in [M] \nonumber \\
    &\text{(By Definition \eqref{eq:power-wasserstein-definition})} \\ 
    &=\sup_{\Prob \in \sP(\sX)} \inf_{\gamma \in \Gamma(\Prob, \ProbHat)} \int_{\sX \times \sX} \norm{x-y}^{\rho} \gamma(dx, dy) \label{eq:auxiliary-proof-sup-problem} \\
    &\quad \text{ s.t. } \Prob(C_i) \in [\evlowerb{i}, \evupperb{i}] \text{ for all } i \in [M], \nonumber \\
    &\text{(As $\ProbHat$ has discrete support)}\\
    &=\sup_{\Prob \in \sP(\sX)} \inf_{\gamma \in \Gamma(\Prob, \ProbHat)} \sum_{j=1}^M\int_{ \sX} \norm{x-c_j}^{\rho} \gamma(dx, c_j) \nonumber\\
    &\quad \text{ s.t. } \Prob(C_i) \in [\evlowerb{i}, \evupperb{i}] \text{ for all } i \in [M], \nonumber \\
        &\text{(As $\{ C_i \}_{i=1}^M$ is a partition of $\sX$ )} \nonumber \\
    &=\sup_{\Prob \in \sP(\sX)} \inf_{\gamma \in \Gamma(\Prob, \ProbHat)} \sum_{i=1}^M \sum_{j=1}^M\int_{ C_i} \norm{x-c_j}^{\rho} \gamma(dx, c_j) \nonumber\\
    &\quad \text{ s.t. } \Prob(C_i) \in [\evlowerb{i}, \evupperb{i}] \text{ for all } i \in [M], \nonumber \\
    &\leq\sup_{\Prob \in \sP(\sX)} \inf_{\gamma \in \Gamma(\Prob, \ProbHat)} \sum_{i=1}^M \sum_{j=1}^M \max_{x\in C_i} \norm{x-c_j}^{\rho}  \gamma(C_i, c_j) \label{Eqn:DiscreteTransportInfimum}\\
    &\quad \text{ s.t. } \Prob(C_i) \in [\evlowerb{i}, \evupperb{i}] \text{ for all } i \in [M]. \nonumber
\end{align}
Now, we note that for any transport plans $\gamma_1,\gamma_2$ such that $\forall i,j \in [M]$ $\gamma_1(C_i,c_j)=\gamma_2(C_i,c_j)$, it holds that they lead to the same value in \eqref{Eqn:DiscreteTransportInfimum}. As a consequence, it also holds that any feasible distributions $\Prob_1, \Prob_2$ such that $\Prob_1(C_i)=\Prob_2(C_i)$ $\forall i \in [M]$ also lead to the same value in \eqref{Eqn:DiscreteTransportInfimum}. Thus, we have that \eqref{Eqn:DiscreteTransportInfimum} can be rewritten as the following supremum of a discrete optimal transport problem, which is a max-min linear program:
\begin{align}\label{Eqn:DiscreteTransportInfimumDiscrete}
     &=\sup_{\omega \in \Delta^M} \inf_{\gamma \in \Delta^{M\times M}} \sum_{i=1}^M \sum_{j=1}^M \max_{x\in C_i} \norm{x-c_j}^{\rho}  \evgamma{i,j} \\
    &\qquad \text{s.t. }\;
\left\{
\begin{array}{l}
\sum_{j=1}^M\evgamma{i,j}=\evomega{i}  \text{ for all } i \in [M] \nonumber \\
\sum_{i=1}^M\evgamma{i,j}=\evpi{j} \text{ for all } j \in [M] \nonumber \\
\evomega{i} \in [\evlowerb{i}, \evupperb{i}] \text{ for all } i \in [M] \nonumber
\end{array}
\right.
\end{align}
where $\evomega{i}$ replaces $\Prob(C_i)$ and we recall that by definition $\ProbHat(c_j)=\evpi{j}$. Therefore, any $\gamma$ that satisfies the above constraints leads to a valid upper bound for \eqref{Eqn:DiscreteTransportInfimumDiscrete}. Now, we observe that among the set of admissible $\gamma$, there always be at least one such that $\forall i\in [M]$, $\evgamma{i,i}\geq\min \big\{ \evomega{i},\evpi{i} \big\}$\footnote{Consider, for instance, the transport plan adapted from Theorem 6.15 in \citet{villani2008optimal} given by $\evgamma{i,j}=\min\{ \evomega{i}, \evpi{j} \}\indicator_{i=j} + \frac{1}{a} (\evomega{i}-\evpi{i})_{+} (\evomega{j}-\evpi{j})_{-}$, where $(x)_{+} = \max\{ x, 0 \}$, $(x)_{-} = \max\{ -x, 0 \}$, and $a = \sum_{i=1}^M (\evomega{i}-\evpi{i})_{+}=\sum_{i=1}^M (\evomega{i}-\evpi{i})_{-}$. We note that $\gamma \in \Gamma(\omega, \pi)$ as $\gamma \geq 0$, $\sum_{j=1}^M \evgamma{i,j}=\evomega{i} - (\evomega{i}-\evpi{i})_{+} + \sum_{j=1}^M \frac{1}{a} (\evomega{i}-\evpi{i})_{+} (\evomega{j}-\evpi{j})_{-} =  \evomega{i}$ (where we use the fact that $\min\{ \evomega{i}, \evpi{i} \}=\evomega{i} - (\evomega{i}-\evpi{i})_{+} \})$, and, analogously, $\sum_{i=1}^M \evgamma{i,j}=\evpi{j}$. For this transport plan, it holds that $\evgamma{i,i}\geq \min\{ \evomega{i}, \evpi{i} \}$.}. Consequently, it holds that
\begin{align}\label{eq:proof-inf-to-sup-conservative}  \eqref{Eqn:DiscreteTransportInfimumDiscrete}&\leq\sup_{\substack{
\omega \in \Delta^M, \\
\gamma \in \Delta^{M \times M}
}} \sum_{i=1}^M \sum_{j=1}^M \max_{x \in C_i} \norm{x-c_j}^{\rho} \evgamma{i,j} \\
    &\qquad \text{s.t. }\;
\left\{
\begin{array}{l}
\gamma \bar\indicator = \omega \nonumber \\
\gamma^\top \bar\indicator = \pi \nonumber \\
\evgamma{i,i} \geq \min \bigl\{ \evomega{i}, \evpi{i} \bigr\},
\quad \forall i \in [M] \nonumber \\
\evomega{i} \in [\evlowerb{i}, \evupperb{i}] \quad \forall i \in [M].\nonumber
\end{array}
\right.
\end{align}
We have thus showed that $\sT_{\rho}(\Prob,\ProbHat) \leq \epsilon(\sD_N)^\rho$. From the Wasserstein definition in \eqref{eq:wasserstein-distance-definition}, we conclude that $\wasserstein_\rho(\Prob,\ProbHat) \leq \epsilon(\sD_N)$.
\end{proof}

We are now ready to prove the main Theorem. 
\begin{proof}
Under the measure $\Prob^N$, by complementarity of events,
\begin{align*}
    \Prob^{N}\Big( \wasserstein_{\rho}(\Prob,\ProbHat) \leq \epsilon(\samplesVector) \Big) &= 1- \Prob^{N}\Big( \wasserstein_{\rho}(\Prob,\ProbHat) >  \epsilon(\samplesVector) \Big)
\end{align*}
Now, by Proposition \ref{prop:wasserstein-bound-no-confidence} and because $\epsilon(\samplesVector)$ only depends on the data through $\lowerb,\upperb$,
\begin{align*}
\Prob^{N}\Big( \wasserstein_{\rho}(\Prob,\ProbHat) >  \epsilon(\samplesVector) \Big) &\leq \Prob^{N}\Big( \exists i\in \{1,...,M\} \text{ s.t. }  \evlowerb{i} > \Prob(C_i) \vee \evupperb{i} < \Prob(C_i).  \Big)\\
& \quad \text{(By the Union bound)}\\
& \leq \sum_{i=1}^M \Big( \Prob^{N}(  \evlowerb{i} > \Prob(C_i)  ) + \Prob^{N}( \evupperb{i} < \Prob(C_i)) \Big).
\end{align*}
Now, to conclude it is enough to notice that each  $\Prob^{N}\Big( \evlowerb{i} > \Prob(C_i)  \Big)$ and $ \Prob^{N}\Big( \evupperb{i} < \Prob(C_i)  \Big)$ represent the tails of a binomial distribution, where $\evlowerb{i}$ and $\evupperb{i}$ are selected using samples $(x_1,...,x_N)\sim \mathbb{P}^{N}$ according to Eqns 1 and 2 in \citep{thulin2014cost}. Thus, from \eqref{eq:clopper-pearson-lower-bound} and \eqref{eq:clopper-pearson-upper-bound}, it holds that:
\begin{align*}
    \Prob^{N}\Big( \wasserstein_{\rho}(\Prob,\ProbHat) > \epsilon(\samplesVector) \Big) \leq \sum_{i=1}^M \Big( \frac{\beta}{2M} + \frac{\beta}{2M} \Big) = \beta
\end{align*}
which concludes the proof.
\end{proof}

\subsubsection{Proof of Proposition \ref{prop:milp-into-lp}}
\begin{proof}
    Consider the same setting as in Theorem \ref{thm:main-theorem}. For any $v \in \Delta^M$, by the triangle inequality, it holds that
    \begin{align*}
        \wasserstein_{\rho}(\Prob,\ProbHat) &\leq \wasserstein_\rho(\Prob, \sum_{i=1}^M \evv{i}\delta_{c_i}) + \wasserstein_\rho(\sum_{i=1}^M \evv{i}\delta_{c_i}, \ProbHat).
    \end{align*}
    In particular, consider $v = \arg\min_{z \in \text{Vert}\big( \constrainedSimplex{\lowerb}{\upperb}^M \big)} \norm{z - \pi}$. Because $v \in \constrainedSimplex{\lowerb}{\upperb}^M$, we denote $I = \{ i \in [M] : \evv{i}=\evlowerb{i} \}$ the set of indexes for which the coordinate $\evv{i}$ assumes the lower bound probability, $J = \{ j \in [M] : \evv{j}=\evupperb{j} \}$, and $f = [M] \setminus (I \cup J)$ as the free index which guarantees that $v$ is a valid probability vector. Then, following the same steps as in the proof of Theorem \ref{thm:main-theorem}, we can upper-bound the first term $\wasserstein_\rho(\Prob, \sum_{i=1}^M \evv{i}\delta_{c_i})=\sT_\rho(\Prob, \sum_{i=1}^M \evv{i}\delta_{c_i})^\frac{1}{\rho}$ by
    \begin{align}
        \sT_\rho(\Prob, \sum_{i=1}^M \evv{i}\delta_{c_i}) &\leq \sup_{\substack{
\omega \in \Delta^M, \\
\gamma \in \Delta^{M \times M}
}} \sum_{i=1}^M \sum_{j=1}^M \max_{x \in C_i} \norm{x-c_j}^{\rho} \evgamma{i,j} \nonumber \\
        &\qquad \text{s.t. }\;
\left\{
\begin{array}{l}
\gamma \bar\indicator = \omega \\
\gamma^\top \bar\indicator = v \\
\evgamma{i,i} \geq \min \bigl\{ \evomega{i}, \evv{i} \bigr\},
\quad \forall i \in [M] \\
\evomega{i} \in [\evlowerb{i}, \evupperb{i}] \quad \forall i \in [M],
\end{array}\label{eq:auxiliary-proof-min-constraint}
\right.
    \end{align}
    To conclude, we observe that because  $\evv{i}=\evlowerb{i}$ for $i \in I$, $\evv{j}=\evupperb{j}$ for $j \in J$, and $\evomega{i} \in [\evlowerb{i}, \evupperb{i}]$, we know in advance which variables activate the minimum in \eqref{eq:auxiliary-proof-min-constraint} for any $i\in I$ and $j\in J$. The rest of the proof follows as in the proof of Theorem \ref{thm:main-theorem}. 
\end{proof}

\subsubsection{Proof of Measurability}\label{proof:measurability}
\begin{proposition}
    The event $\big\{ \wasserstein_{\rho}(\Prob,\ProbHat) \leq \epsilon(\samplesVector) \big\}$ is $\borel(\sX^N)$-measurable.
\end{proposition}
\begin{proof}
Consider $z(\samplesVector)=\wasserstein_{\rho}(\Prob,\ProbHat)-\epsilon(\samplesVector)$, then we want to show that the event $\big\{z(\samplesVector) \leq 0 \big\}$ is $\borel(\sX^N)$-measurable. As the zero sublevel set of a measurable function is measurable, it is enough to show that $z:\sX^N\to \realNum$ is measurable. This is guaranteed if we show that both $\wasserstein_{\rho}(\Prob,\ProbHat)$, and $\epsilon(\samplesVector)$ are measurable. Note that both $\wasserstein_{\rho}(\Prob,\ProbHat)$, and $\epsilon(\samplesVector)$ depend on data $\samplesVector$ uniquely through the count of samples within each region $C_i$, i.e. $\sum_{n=1}^{N} \indicator_{C_i}(\vx_n)$ (the first via $\ProbHat=\sum_{i=1}^M \Big(\frac{1}{N}\sum_{n=1}^{N} \indicator_{C_i}(\vx_n) \Big)  \delta_{c_i}$, while the others via $\lowerb$ and $\upperb$, as in \eqref{eq:clopper-pearson-lower-bound} and \eqref{eq:clopper-pearson-upper-bound}). Consequently, when seen as a function of $\samplesVector$, it holds that  all two functions are constant for any set $C_{i_1}\times ... \times C_{i_M} \subseteq \mathcal{X}^N$, where each $i_k\in\{1,\dots,M\}$, that is, $\wasserstein_{\rho}(\Prob,\ProbHat):\mathcal{X}^N \to \mathbb{R}$, and $\epsilon:\mathcal{X}^N \to \mathbb{R}$ are piecewise constant with a finite number of pieces. Under the assumption that $\{ C_1,\dots,C_M \}$ is a partition of $\sX$ in $\borel(\sX)$-measurable sets, $C_{i_1}\times ... \times C_{i_M}$ is measurable and, therefore, $\wasserstein_{\rho}(\Prob,\ProbHat)$, and $\epsilon$ are measurable (e.g., from \citet{rudin1987real}, Definition 1.16).
\end{proof}

\subsubsection{Proof of Proposition \ref{prop:convergence-bound-n-m}}

First, we start by proving the following auxiliary lemma.
\begin{lemma}\label{lemma:bound-eps-by-eps1-eps2}
Let $\epsilon(\samplesVector)$ be as defined in Theorem \ref{thm:main-theorem} and $\xi(\samplesVector)$ as in Proposition \ref{prop:milp-into-lp}. Then, it holds that $\epsilon(\samplesVector)^\rho \leq \epsilon_1(\samplesVector) + \epsilon_2(\samplesVector)$ and $\xi(\samplesVector)^\rho \leq \epsilon_1(\samplesVector) + \epsilon_2(\samplesVector)$, where
\begin{align*}
    \epsilon_1(\samplesVector) &= \sum_{i=1}^M  \max_{x\in C_i}\norm{x-c_i}^{\rho}\evupperb{i}
\end{align*}
and
\begin{align*}
    \epsilon_2(\samplesVector) &= \norm{\sX}^\rho \sum_{i=1}^M (\evupperb{i} - \evlowerb{i})
\end{align*}
\end{lemma}
\begin{proof}
    The objective functions in \eqref{eq:def-epsilon} and \eqref{eq:def-xi} can be rewritten as
    \begin{align*}
        \sum_{i=1}^M \sum_{j=1}^M \max_{x \in C_i} \norm{x-c_j}^{\rho} \evgamma{i,j} &= \sum_{i=1}^M \max_{x \in C_i} \norm{x-c_i}^{\rho} \evgamma{i,i} + \sum_{i=1}^M \sum_{j\neq i} \max_{x \in C_i} \norm{x-c_j}^{\rho} \evgamma{i,j}.
    \end{align*}
    Given the constraints in \eqref{eq:def-epsilon} and \eqref{eq:def-xi}, because $\evgamma{i,i} \geq \min\{ \evomega{i}, \evpi{i} \}$, and both $\evomega{i} \leq \evupperb{i}$ and $\evpi{i} \leq \evupperb{i}$, it holds that $\sum_{i=1}^M \max_{x \in C_i} \norm{x-c_i}^{\rho} \evgamma{i,i} \leq \epsilon_1(\samplesVector)$ for any feasible $\omega, \gamma$ in \eqref{eq:def-epsilon} and \eqref{eq:def-xi}. For the second term, we observe that for any feasible $\omega, \gamma$ in \eqref{eq:def-epsilon} and \eqref{eq:def-xi}
    \begin{align*}
        \sum_{i=1}^M \sum_{j\neq i} \max_{x \in C_i} \norm{x-c_j}^{\rho} \evgamma{i,j} &\leq \sum_{i=1}^M \max_{j \neq i}\max_{x \in C_i} \norm{x-c_j}^{\rho} \sum_{j\neq i} \evgamma{i,j} \\
        &\quad (\text{As } \sum_{j=1}^M \evgamma{i,j}=\evomega{i} \text{ and } \evgamma{i,i} \geq \min\{ \evomega{i}, \evpi{i} \}) \\
        &\leq \sum_{i=1}^M \max_{j \neq i}\max_{x \in C_i} \norm{x-c_j}^{\rho} (\evomega{i} - \min\{ \evomega{i}, \evpi{i} \}) \\
        &\leq \sum_{i=1}^M \max_{j \neq i}\max_{x \in C_i} \norm{x-c_j}^{\rho} (\evupperb{i} - \evlowerb{i}) \\
        &\leq \norm{\sX}^\rho \sum_{i=1}^M (\evupperb{i} - \evlowerb{i}),
    \end{align*}
    which concludes the proof.
\end{proof}
We are now ready to prove Proposition \ref{prop:convergence-bound-n-m}.
\begin{proof}
    From Lemma \ref{lemma:bound-eps-by-eps1-eps2}, we can show the convergence to zero of $\epsilon(\samplesVector)$ in Theorem \ref{thm:main-theorem} and $\xi(\samplesVector)$ in Proposition \ref{prop:milp-into-lp} by proving the convergence for both $\epsilon_1(\samplesVector)$ and $\epsilon_2(\samplesVector)$. Then, we also need to prove the convergence for the residual term $\wasserstein_\rho(\sum_{i=1}^M \evv{i}\delta_{c_i}, \ProbHat)$ in Proposition \ref{prop:milp-into-lp}. We start with $\epsilon_2(\samplesVector)$. From Lemma \ref{lemma:bound-eps-by-eps1-eps2},
    \begin{align}
    \epsilon_2(\samplesVector) &= \norm{\sX}^\rho \sum_{i=1}^M (\evupperb{i}-\evlowerb{i})
    \end{align}

    From Theorem 1 in \citet{thulin2014cost}, for region $C_i$ assuming $0<\evpi{i}=\frac{1}{N}\sum_{n=1}^{N} \indicator_{C_i}(\vx_i)<1$, for large $N$, one can approximate up to order $O(N^{-\frac{3}{2}})$: 
    \begin{align*}
    \evupperb{i}-\evlowerb{i} \approx \frac{2}{\sqrt{N}}\text{Probit}\Big( 1-\frac{\beta}{2M} \Big)\sqrt{\evpi{i}(1-\evpi{i})} + \frac{1}{N}.
    \end{align*}
    Instead, if $\evpi{i}=0$ or $\evpi{i}=1$, we simply have:
    \begin{align*}
    \evupperb{i}-\evlowerb{i} = 1- {\Big( \frac{\beta}{2M}\Big)}^{\frac{1}{N}}.
    \end{align*}
    In the following, we refer to $I_0$ as the indexes of regions with no samples, i.e. $\evpi{i}=0$ for $i\in I_0$, and $I_{+}$ the remaining regions (i.e., $0 < \evpi{i} < 1$ for $i\in I_{+}$). In this case, from above, up to order $O(N^{-\frac{3}{2}})$:
    \begin{align}
        \sum_{i=1}^M (\evupperb{i}-\evlowerb{i}) &= \sum_{i\in I_0} (\evupperb{i}-\evlowerb{i}) + \sum_{i\in I_{+}} (\evupperb{i}-\evlowerb{i}) \nonumber \\
        &\approx \sum_{i \in I_0} \Big( 1- \Big( \frac{\beta}{2M}\Big)^{\frac{1}{N}} \Big) + \sum_{i \in I_{+}} \Big( \frac{2}{\sqrt{N}}\text{Probit}\Big( 1-\frac{\beta}{2M} \Big)\sqrt{\evpi{i}(1-\evpi{i})} + \frac{1}{N} \Big) \label{eq:o-notation-main}
    \end{align}
    Let's examine the first term:
    \begin{align}
        \sum_{i \in I_0} \Big( 1- \Big( \frac{\beta}{2M}\Big)^{\frac{1}{N}} \Big) &= |I_0| \Big( 1- \Big( \frac{\beta}{2M}\Big)^{\frac{1}{N}} \Big) \leq M \Big( 1- \Big( \frac{\beta}{2M}\Big)^{\frac{1}{N}} \Big) = O\Big( \frac{M \log M}{N} \Big) \label{eq:o-notation-exp}
    \end{align}
    where the last equality can be obtained using Taylor expansion of the exponential, i.e. $1-\Big( \frac{\beta}{2M}\Big)^{\frac{1}{N}}=1-e^{-\frac{\log(2M / \beta)}{N}} = 1 - \Big( 1 - \frac{\log(2M / \beta)}{N} + O\Big( \frac{\log(2M / \beta)^2}{N^2} \Big) \Big)= O\Big( \frac{\log M}{N} \Big)$.
    
    Thus, for $\alpha<1$ and $M=O(N^\alpha)$, we have that, from \eqref{eq:o-notation-exp}:
    \begin{align}
        \sum_{i=1}^{M_0} \Big( 1- \Big( \frac{\beta}{2M}\Big)^{\frac{1}{N}} \Big) &= O\Big( \frac{N^\alpha \log N^\alpha}{N} \Big)=O\Big( \frac{\alpha \log N}{N^{1-\alpha}} \Big) \underset{N \to \infty}{\longrightarrow} 0
    \end{align}
    Under the same assumptions, we focus on the third term in \eqref{eq:o-notation-main}. We have:
    \begin{align}
        \sum_{i \in I_{+}} \frac{1}{N} = \frac{|I_{+}|}{N} \leq \frac{M}{N} = O \Big( \frac{M}{N} \Big)= O\Big( \frac{1}{N^{1-\alpha}} \Big)\underset{N \to \infty}{\longrightarrow} 0 
    \end{align}
    Finally, we examine the second term in \eqref{eq:o-notation-main}.
    \begin{align}
        \sum_{i \in I_{+}} \frac{2}{\sqrt{N}}\text{Probit}\Big( 1-\frac{\beta}{2M} \Big)\sqrt{\evpi{i}(1-\evpi{i})} &\leq \frac{2}{\sqrt{N}}\text{Probit}\Big( 1-\frac{\beta}{2M} \Big) \sum_{i=1}^{M} \sqrt{\evpi{i}(1-\evpi{i})} \nonumber \\
        &\leq \frac{2}{\sqrt{N}}\text{Probit}\Big( 1-\frac{\beta}{2M} \Big) \sum_{i=1}^{M} \sqrt{\frac{1}{M}\Big( 1-\frac{1}{M} \Big)} \label{eq:o-notation-concave} \\
        &= \frac{2}{\sqrt{N}}\text{Probit}\Big( 1-\frac{\beta}{2M} \Big) \frac{M \sqrt{M-1}}{M} \nonumber \\
        &= \frac{2}{\sqrt{N}}\text{Probit}\Big( 1-\frac{\beta}{2M} \Big) \sqrt{M-1} \nonumber
    \end{align}
    where the second inequality in \eqref{eq:o-notation-concave} uses the fact that the maximum of $\sum_{i=1}^{M} \sqrt{\evpi{i}(1-\evpi{i})}$ subject to $\pi \in \Delta^M$ is attained when $\evpi{i}=\frac{1}{M}$. Indeed, by the Cauchy-Schwarz inequality \citep{steele2004cauchy}, 
    $$\Big( \sum_{i=1}^{M} \sqrt{\evpi{i}(1-\evpi{i})} \Big)^2 \leq \Big( \sum_{i=1}^{M} \evpi{i} \Big) \Big( \sum_{i=1}^{M} (1-\evpi{i}) \Big) = M-1$$
    as $\sum_{i=1}^M \evpi{i}=1$. Thus, $\sum_{i=1}^{M} \sqrt{\evpi{i}(1-\evpi{i})} \leq \sqrt{M-1}$. To conclude, we observe that this upper-bound is attained for $\evpi{i}=\frac{1}{M}$.
    Further, we note that $\text{Probit}(1-\frac{\beta}{2M})=O \Big( \sqrt{\log M} \Big)$ (Eqn (2) in \citet{blair1976rational}). In this case,
    \begin{align}
        \sum_{i \in I_{+}} \frac{2}{\sqrt{N}}\text{Probit}\Big( 1-\frac{\beta}{2M} \Big)\sqrt{\evpi{i}(1-\evpi{i})} &= O\Big( \sqrt{\frac{M \log M}{N}} \Big)=O\Big( \sqrt{\frac{ \alpha \log N}{N^{1-\alpha}}} \Big)\underset{N \to \infty}{\longrightarrow} 0 
    \end{align}
    To conclude, we remark that the remainder of the approximation in \eqref{eq:o-notation-main} also converges to zero, i.e. $\sum_{i=1}^M O(N^{-\frac{3}{2}})=O(M N^{-\frac{3}{2}})=O(N^{-\frac{3}{2}+\alpha})\underset{N \to \infty}{\longrightarrow} 0$. Now, we move to $\epsilon_1(\samplesVector)$. Again, from \eqref{lemma:bound-eps-by-eps1-eps2},
    \begin{align*}
    \epsilon_1(\samplesVector)&=\sum_{i=1}^M  \max_{x\in C_i}\norm{x-c_i}^{\rho}\evupperb{i} \\
        &= \sum_{i=1}^M \max_{x\in C_i}\norm{x-c_i}^{\rho} \evlowerb{i}+ \sum_{i=1}^M \max_{x\in C_i}\norm{x-c_i}^{\rho} (\evupperb{i}-\evlowerb{i}) \\
        &\leq \sum_{i=1}^M \max_{x\in C_i}\norm{x-c_i}^{\rho}\evpi{i}+ \norm{\sX}^{\rho}\sum_{i=1}^M (\evupperb{i}-\evlowerb{i})
    \end{align*}
    The convergence of the second term $\norm{\sX}^{\rho}\sum_{i=1}^M (\evupperb{i}-\evlowerb{i})$ was studied for $\epsilon_2(\samplesVector)$. Therefore, to conclude, it is enough to prove the convergence of $\sum_{i=1}^M \max_{x\in C_i}\norm{x-c_i}^{\rho}\evpi{i}$. From the assumptions on the partition,
    \begin{align*}
    \sum_{i=1}^M \max_{x\in C_i}\norm{x-c_i}^{\rho} \evpi{i} &\leq
        \sum_{i=1}^M \norm{C_i}^{\rho}\evpi{i} \\
        &= \sum_{i=1}^{M-1} \norm{C_i}^{\rho}\evpi{i} + \norm{C_M}^{\rho}\evpi{M} \\
        &\leq \frac{L}{M}\sum_{i=1}^{M-1}\evpi{i} + \norm{C_M}^{\rho} \frac{\delta^\rho}{2\norm{\sX}^\rho} \\
        &\leq \frac{L}{M} + \frac{\delta^\rho}{2} \leq \frac{L}{N^\alpha} +  \frac{\delta^\rho}{2} \underset{N \to \infty}{\longrightarrow} \frac{\delta^\rho}{2}
    \end{align*}
    Because of the convergence to zero of all terms above, there exists a $N^* \in \natNum$ such that for any $N \geq N^*$, their collection is upper-bounded by $\frac{\delta^\rho}{4}$. Thus, for $N\geq N^*$, $\epsilon(\sD_N) \leq \left( \frac{3}{4} \right)^\frac{1}{\rho}\delta \leq \delta$ and, analogously, $\xi(\sD_N) \leq \left( \frac{3}{4} \right)^\frac{1}{\rho}\delta$. To conclude the proof, the only missing step is to prove the convergence of the term $\wasserstein_\rho(\sum_{i=1}^M \evv{i}\delta_{c_i}, \ProbHat)$ to zero (as this is equivalent to the existence of a $N^*$ for which $N\geq N^*$ implies $\wasserstein_\rho(\sum_{i=1}^M \evv{i}\delta_{c_i}, \ProbHat) \leq \left( 1 - \left( \frac{3}{4} \right)^\frac{1}{\rho} \right) \delta$). This follows from observing that $\wasserstein_\rho(\sum_{i=1}^M \evv{i}\delta_{c_i}, \ProbHat)=\sT(\sum_{i=1}^M \evv{i}\delta_{c_i}, \ProbHat)^\frac{1}{\rho}$ and
    \begin{align}\label{eq:proof-extra-term-convergence}
       \sT_\rho(\sum_{i=1}^M \evv{i}\delta_{c_i}, \ProbHat) &\leq \sum_{i=1}^M \norm{c_i-c_i}^\rho \evgamma{i,i} + \sum_{i=1}^M \sum_{i\neq j} \norm{c_i-c_j}^\rho \evgamma{i,j}
    \end{align}
    for the transport plan $\gamma \in \Gamma(v, \pi)$ such that $\evgamma{i,i}\geq \min\{ \evv{i}, \evpi{i} \}$. Following the same reasoning from the proof of Lemma \ref{lemma:bound-eps-by-eps1-eps2}, we can upper-bound \eqref{eq:proof-extra-term-convergence} by
    \begin{align*}
        \eqref{eq:proof-extra-term-convergence} &\leq \norm{\sX}^\rho \sum_{i=1}^M (\evupperb{i} - \evlowerb{i})
    \end{align*}
    for which the convergence was already previously studied.
\end{proof}


\end{document}